\documentclass[12pt,letterpaper]{article}
\usepackage{setspace}
\usepackage[margin=1.2in]{geometry}

\usepackage{amsthm, amsmath, amssymb, xcolor}
\usepackage[hidelinks]{hyperref}

\theoremstyle{plain}
\newtheorem{theorem}{Theorem}
\newtheorem{proposition}{Proposition}
\newtheorem{lemma}{Lemma}

\theoremstyle{definition}
\newtheorem{definition}{Definition}
\newtheorem*{remark*}{Remark}
\theoremstyle{plain}
\newtheorem{applemma}{Lemma}[section]
\newtheorem{appproposition}[applemma]{Proposition}
\newtheorem{apptheorem}[applemma]{Theorem}
\newtheorem{appcorollary}[applemma]{Corollary}
\theoremstyle{definition}
\newtheorem{appdefinition}[applemma]{Definition}
\newtheorem{appremark}[applemma]{Remark}

\newcommand{\eos}{\mathtt{eos}}
\newcommand{\dTV}{d_{\textup{\textsc{tv}}}}
\newcommand{\Pcal}{\mathcal{P}}
\newcommand{\Qcal}{\mathcal{Q}}
\newcommand{\Scal}{\mathcal{S}}
\newcommand{\Vcal}{\mathcal{V}}
\newcommand{\X}{\mathcal{X}}
\newcommand{\PP}{\mathbb{P}}
\newcommand{\E}{\mathbb{E}}
\newcommand{\Error}{\mathrm{Error}}
\newcommand{\Loss}{\mathrm{Loss}}
\newcommand{\Qfull}{\Qcal_{\mathrm{full}}}
\newcommand{\Qfloor}{\Qcal_{\mathrm{floor}}}
\newcommand{\QT}{\Qcal_{\mathrm{Transformer}}}
\newcommand{\hold}{\mathrm{hold}}

\title{The Hardness of Estimating Cross-Entropy Risk\\ for Language Models}

\author{Hanti Lin \\[0.5em] University of California, Davis \\ika@ucdavis.edu}

\date{}
\begin{document}

\maketitle

\begin{abstract} \noindent
Language models, once pretrained, are routinely assessed by the holdout method, which estimates their per-token cross-entropy risk---a quantity that may be finite or infinite. We prove possibility and impossibility theorems that characterize the hardness of this estimation problem. The estimand depends on two objects: the unknown data distribution $p$ and the distribution $q$ induced by the language model under study. Our ignorance of $p$ is empirical, whereas our ignorance of $q$, or at least of its tail, is mathematical; so the parameter space {\em should} range over pairs $(p,q)$---a product space---to fully capture our ignorance. {\em Negative results}: no estimator is pointwise consistent, even when the parameter space is restricted to full-support models and to data distributions whose expected text length is below a fixed bound; worse, every estimator fails on a dense set of pairs. {\em Positive results}: the holdout estimator becomes pointwise consistent when the models are further restricted to Transformers or, architecture-neutrally, to full-support models with a finite context window; and, where no such design is in place, it still consistently estimates the risk capped at any fixed threshold, which suffices for screening candidate models in practice. {\em Lesson}: designs such as the Transformer architecture and finite context windows, although adopted for computational or engineering reasons, pay a {\em statistical dividend}: they suffice to make pointwise consistency possible---a dividend visible only in how they turn a negative result into positive ones.
\end{abstract}

% =====================================================================
\section{Introduction}\label{sec-intro}

Once a language model has been pretrained, the usual first check on it is the holdout method: compute its log loss on held-out texts, averaged per token, and use that number to screen candidate models before anything downstream, such as supervised fine-tuning (Kaplan et al., 2020; Hoffmann et al., 2022). The idea behind it is the law of large numbers: the holdout average converges to the model's per-token cross-entropy risk.  But here is the thing: the risk may be infinite. To be sure, the law of large numbers still ``works'' then: it makes the holdout average converge to infinity, in the sense familiar from real analysis. Yet this does not save the holdout method from an embarrassment. For a model that gives every text positive probability, as softmax models do, the holdout estimate is always a finite number at any sample size, and the difference between a finite estimate and an infinite estimand is infinite. So, in an unlucky situation where the risk---the estimand---is infinite, the holdout method always produces a finite estimate and thus incurs an infinite error---with probability 1. Nor is an infinite risk a remote possibility: for many models it can be obtained by modifying an arbitrarily small tail of the data distribution, and situations with infinite risk lie arbitrarily close to every situation, as we will see below (look for the keyword ``dense'').

Our goal is to study the asymptotic performance of the holdout method and of any other procedure for estimating the per-token cross-entropy risk. The kind of performance we have in mind is pointwise consistency: convergence in probability to the true value of the estimand at each ``point'' of a certain space---at each possible situation in a parameter space constrained only by background knowledge or assumptions. We start with very weak assumptions, and hence a large parameter space, and prove an impossibility result: no estimator achieves pointwise consistency, and the holdout method is no exception. Next we see how positive results can be restored as appropriate assumptions are added to constrain the parameter space. In particular, pointwise consistency is restored once we assume, or know, that the model under study is a Transformer, or has a finite context window and a softmax output. Interestingly, finite context windows were adopted to contain the computational cost of attention, and the Transformer to replace the sequential computation of recurrent networks with parallel attention (Vaswani et al., 2017)---both were motivated by computational considerations, but our results show that both pay a \emph{statistical dividend}: these designs suffice to make an important measure of model performance estimable in principle---a dividend visible only when the negative and the positive results are put together.

More concretely, the negative results escalate. Under very weak assumptions, no estimator of the per-token risk is pointwise consistent. Worse, none is pointwise consistent even for full-support models, as produced by softmax output layers, and for data whose expected text length is below a fixed bound---an assumption itself already untestable. Even worse, consistency failures are not confined to exotic corners: in topological parlance, every estimator is doomed to fail on a dense set of possible situations; that is, every possible situation has situations arbitrarily nearby, too close to tell apart from data, at which the estimator fails to converge to the true value of the estimand (made precise in Section~\ref{sec-hard}). All of this holds even if the estimator is granted access to every mathematical property of the model, as if it had an oracle for mathematics. 

The negative result is actually meant to do something positive, just as Arrow's (1951) impossibility theorem in social choice has been used to organize the positive escapes from it---proposals for designing systems of voting or preference aggregation (List, 2022). In our case, we consider the following escapes. \emph{Way out 1: restricting the class of models.} A floor---a positive lower bound on next-token conditional probabilities---makes the holdout estimator pointwise consistent. The architecture of Transformers can be shown to secure a floor, by a proof adapted from Du et al.\ (2023). Or, to escape in an architecture-neutral way, a floor can also be secured with a finite context window---an external design that can be applied to any architecture (Section~\ref{sec-models}). \emph{Way out 2: modifying the estimand.} The risk reported only up to a threshold---the capped risk---is consistently estimated at every possible situation: a way out by a capped estimand, which is enough for the practice of screening candidate models (Section~\ref{sec-estimand}).

Stating our results with conceptual clarity requires a slight augmentation of the usual notion of pointwise consistency (Casella and Berger, 2002, Definition~10.1.1). The estimand, the cross-entropy risk, is defined---whenever it is defined at all---at a pair $(p,q)$ of a data distribution $p$ and the predictive distribution $q$ induced by the model under study. So consistency is defined at a pair, and pointwise consistency is most naturally defined by quantifying over a range of pairs---a range of possibilities that represents our ignorance, or the hedged assumptions we are willing to make. That ignorance goes both ways. It is empirical about $p$, the dimension that a parameter space conventionally represents; and it is mathematical about $q$: even if we know the architecture and the weights of the model under study, we may not know exactly which predictive distribution it induces, at least in its tail---whether its log loss grows quadratically on long texts, say. So we define pointwise consistency by universal quantification over a product space of pairs $(p,q)$, with estimators taking the model as well as the data as input; this preserves the original spirit of a parameter space. It also matches the spirit of the holdout method itself: one procedure applied across many models, as when a single estimator charts the losses behind a scaling law.

Earlier work leaves the coordinate $q$ implicit, treating the model $q$ under study as a free variable, and that is fine for its purpose: stating positive results, which are meant to hold for any model. But our negative result says that any estimator fails for {\em some} model with data generated from {\em some} distribution, and an existential quantifier cannot be left implicit: its range must be stated. This is crucial for being clear about where the negative result still haunts and where a positive result is gained---over all full-support language models, say, or over just the Transformer-based ones. Moreover, the parameter space can shrink even while we study the same model: proving a theorem about the tail of the distribution $q$ that the model induces narrows the range of $q$---which is how architecture theorems become statistical guarantees (Section~\ref{sec-models}).

To recap, here is a list of our contributions, which also serves as a roadmap:
\begin{itemize}\setlength{\itemsep}{0pt}\setlength{\parskip}{0pt}
\item an augmented notion of pointwise consistency over pairs of a data distribution and a model, whose parameter space represents mathematical as well as empirical ignorance (Section~\ref{sec-setting});
\item impossibility theorems for estimating the per-token risk, extending topological methods hitherto used for hypothesis testing, via constructions showing that situations with finite risk and situations with infinite risk lie arbitrarily close to each other (Section~\ref{sec-hard});
\item a bridge from architecture theorems to statistical guarantees, yielding pointwise consistency for Transformers and for full-support models with a finite context window, and thereby revealing the statistical dividend of designs adopted for computational or engineering reasons (Section~\ref{sec-models});
\item a proof that a modest modification of the estimand---capping the risk at a threshold---restores pointwise consistency at every situation at which the risk is defined (Section~\ref{sec-estimand}).
\end{itemize}
Section~\ref{sec-related} reviews related work, and Section~\ref{sec-discussion} wraps up, discussing what the results imply and what they leave open. All proofs are in the appendices.

% =====================================================================
\section{Related Work}\label{sec-related}

\paragraph{Enriched parameter spaces in learning theory.}
The textbook holdout bound for a bounded loss already ranges, implicitly, over pairs of a distribution and a predictor: it holds whatever the distribution and whatever the predictor (Shalev-Shwartz and Ben-David, 2014, Theorem~11.1), so the range of predictors never needs to be stated. Under the unbounded log loss no such uniform bound is available, and because our main results are negative and require existential quantification, the range of models must be stated explicitly (Section~\ref{sec-intro}).

\paragraph{Mathematical uncertainty.}
That a model we have constructed and used can be possessed yet not fully understood is a familiar thought. Asmussen et al.\ (1992) prove negative results for the initial transient problem in simulation---a problem closely related to convergence diagnostics for Markov chain Monte Carlo---in which, as here, the tail of a distribution plays the critical role. There is indeed existing work that, like ours, treats mathematical uncertainty about a model under study together with ignorance about a data-generating process: Bayesian calibration of computer models, which treats a model's output as uncertain until it is computed (Kennedy and O'Hagan, 2001), and probabilistic numerics more broadly (Diaconis, 1988; Hennig et al., 2015). However, that work is Bayesian, placing priors on mathematical objects. We offer what we believe to be the first frequentist treatment: mathematical ignorance enters the parameter space alongside empirical ignorance, and proving theorems about the model shrinks the space.

\paragraph{Impossibility and testability.}
For the negative results, we follow the spirit of Bahadur and Savage (1956) and Donoho (1988), though we do not share their topic or technique. Closer is a lineage that links the performance of inference procedures to topology. It begins with the forcing arguments of Gold (1967) and Putnam (1965) in algorithmic learning theory, which is like hypothesis testing in a deterministic world, where i.i.d.\ sampling makes no sense; Kelly (1996) systematized these arguments into topological characterizations of learning performance. The characterizations were then carried over to the stochastic, i.i.d.\ setting by Dembo and Peres (1994), Genin and Kelly (2017), Ermakov (2017), and Boeken et al.\ (2026), while Shah and Peters (2020) identified topological density as the engine of untestability. A key lemma in our proofs (Lemma~\ref{lem-vacillation}) is, in substance, a consequence of the characterization of Genin and Kelly (2017). That work concerns hypothesis testing. We extend the link to estimation, through a reduction: at least for cross-entropy risk, consistent estimation of the risk requires a consistent test of whether it is finite.

\paragraph{Theory of language models.}
Du et al.\ (2023) characterize tightness---whether an autoregressive model halts with probability one---and prove that Transformers are tight. We grant tightness, and borrow part of their proof to show that Transformers have a floor: a positive lower bound on next-token probabilities. Mathematical knowledge of such a floor narrows the product parameter space in a way that escapes our negative result, leading to a positive result (Section~\ref{sec-models}).

A fuller presentation of related work, including scoring rules, model selection, and estimation on countable alphabets, is in Appendix~\ref{app-related}.

% =====================================================================
\section{Setting}\label{sec-setting}

Let $\Vcal$ be a finite vocabulary of tokens, containing an end-of-sequence token $\eos$, and let $\Vcal_0 := \Vcal\setminus\{\eos\}$ be nonempty. A {\em text} is a finite string of tokens from $\Vcal_0$ followed by $\eos$; let $\X$ be the set of all texts, which is countably infinite. For any text $x \in \X$, write $|x|$ for its length, the number of its tokens with $\eos$ included. A {\em context} is a finite string $s = (s_1, \ldots, s_k)$ of tokens from $\Vcal_0$, with $k \geq 0$ so that the empty context is included; let $\Vcal_0^*$ be the set of contexts.

In practice a language model is built from next-token conditional probabilities $q(v\mid s)$---a distribution over the next token $v \in \Vcal$ for each context $s \in \Vcal_0^*$---which together define a joint distribution over texts, assigning probability $q(x) = \prod_{t=1}^{|x|} q(x_t \mid x_{<t})$ to each text $x = (x_1,\ldots,x_{|x|})$, where $x_{<t}$ is the context formed by its first $t-1$ tokens. Strictly speaking, these conditionals define a probability distribution over texts only if the model is {\em tight}: no probability leaks to infinitely long sequences. One normally wants a generator to halt with probability one, and whether an architecture guarantees this is the question Du et al.\ (2023) study; we take it as settled. So we consider only probability distributions over the set $\X$ of texts, setting aside conditionals that fail to be tight. A {\em $($language$)$ model} is a probability distribution $q$ over texts---its {\em predictive distribution}, in the forecaster's sense (Gneiting and Raftery, 2007) rather than the Bayesian one---used to predict texts and their next tokens; let $\Qcal$ be the set of all such models. Throughout, held-out texts are drawn i.i.d.\ from an unknown {\em data distribution} $p$ over texts; let $\Pcal$ be the set of all such distributions.

\paragraph{The per-token risk.}
The {\em per-token cross-entropy risk} of model $q$ under data distribution $p$ is
\[
R(p,q) \;:=\; \frac{\E_p[-\log q(X)]}{\E_p|X|},
\]
the expected loss per expected token. It is what practice reports, as total loss over total tokens, or as the perplexity $e^{R}$. Write $\bar\ell(p) := \E_p|X|$ for the expected text length (the denominator of the above definition). The risk is defined on $\Pcal_{\,\bar\ell<\infty} := \{p \in \Pcal : \bar\ell(p) < \infty\}$, where it takes values in $[0,\infty]$.

Infinite risk is easy to come by. To illustrate, take a vocabulary that has only a single ordinary token $\mathtt{a}$, so that a text always takes the simple form $(\mathtt{a},\ldots,\mathtt{a},\eos)$ and can be identified with its length $\ell \geq 1$. Let the true probability be $p(\ell) \propto \ell^{-3}$ and the predicted probability be $q(\ell) \propto e^{-\ell^2}$, so that its log loss, $-\log q(\ell) = \ell^2 + \mathrm{const}$, grows quadratically. The expected length is then finite, by the Basel sum $\sum_\ell \ell^{-2} = \pi^2/6$ (thanks to Euler); the expected loss, which is $\sum_\ell p(\ell)\,(\ell^2 + \mathrm{const})$, is infinite, by the divergence of the harmonic series $\sum_\ell \ell^{-1}$. It is important to note that not every detail of these two distributions matters---only their tails do: any $p$ with a tail like $\ell^{-3}$ and any $q$ whose log loss grows like $\ell^2$ give the same verdict, and changing finitely many probabilities changes nothing. A sample, however, shows only finitely many lengths, so it says little about the tail. That is the intuition behind the hardness established in Section~\ref{sec-hard}.

\paragraph{The parameter space.}
A parameter space models our limited knowledge, empirical and, as here, mathematical. Our parameter spaces are therefore sets of pairs $(p,q)$ of a data distribution and a predictive distribution---subsets of the product $\Pcal_{\,\bar\ell<\infty}\times\Qcal$---for the reasons given in Section~\ref{sec-intro}. The risk is defined at pairs. And our ignorance of the model is mathematical: how fast its loss grows on ever longer texts is settled only by mathematical proof, and, as the example above shows, this tail property decides whether the risk can be infinite. Moreover, holdout is one procedure applied to every model; the same estimator charts the scatterplot that reveals a scaling law (Kaplan et al., 2020; Hoffmann et al., 2022). So its guarantee should range over a class of models. We call the points $(p,q)$ of a parameter space {\em states of the world}, or {\em states} for short, or sometimes {\em situations} for a more intuitive feel. A range of such pairs represents what we are uncertain about: the data distribution and, mathematically, the model. To shrink the range of models, then, is to prove theorems about the models under study (Section~\ref{sec-models}).

\paragraph{Estimators.}
A {\em holdout sample} of size $n$ is a sequence $y = (y^{(1)},\ldots,y^{(n)})$ of texts, the realization of texts $Y^{(1)},\ldots,Y^{(n)}$ drawn i.i.d.\ from $p$, independently of whatever produced the model. An {\em estimator} is a sequence $\hat R = (\hat R_n)_{n\geq 1}$ of maps $\hat R_n : \X^n \times \Qcal \to [0,\infty]$, taking a sample and the model under study. The {\em holdout estimator} reports total loss over total tokens,
\[
\hat R^{\hold}_n(y;q) \;:=\; \frac{\sum_{j=1}^n -\log q\big(y^{(j)}\big)}{\sum_{j=1}^n \big|y^{(j)}\big|};
\]
so it reads $q$ only at the texts it happens to observe. We allow arbitrary dependence on $q$ deliberately, so that no future estimator exploiting more of $q$ escapes the negative results.\footnote{In practice an estimator receives weights rather than $q$. Our proofs fix a model and vary $p$, so they apply verbatim to estimators that take any representation of the model as input, as long as every model in the parameter space has one; randomized estimators are covered as well (Appendix~\ref{app-prelim}).}

\paragraph{Error.}
Both estimate and estimand may be infinite, and we measure error by taking the straightforward difference; for example, the difference between two real numbers is just the ordinary difference, and the difference between a real number and $\infty$ is infinite. Turning this into a formal definition, the error of an estimate $u$ for an estimand $v$ is
\[
\Error(u,v) := \begin{cases} |u-v| & \text{if } u,v<\infty,\\ \infty & \text{if exactly one is } \infty,\\ 0 & \text{if both are } \infty,\end{cases}
\]
which is an extended metric on $[0,\infty]$ (Appendix~\ref{app-error}). So the only estimate within finite error of an infinite risk is $\infty$ itself: an estimator whose error tends to zero must eventually say whether the risk is finite. An alternative definition of error that makes error always bounded is discussed after Theorem~\ref{thm-main}.

\begin{definition}\label{def-main}
Let $\PP_{p,n}$ denote the law of an i.i.d.\ sample $Y = (Y^{(1)},\ldots,Y^{(n)})$ from $p$.\par
$(i)$ An estimator $\hat R$ is \textbf{consistent at} a state $(p,q) \in \Pcal_{\,\bar\ell<\infty}\times\Qcal$ iff, for every $\epsilon > 0$,
\[
\PP_{p,n}\Big(\Error\big(\hat R_n(Y;q),\,R(p,q)\big) < \epsilon\Big) \to 1 \ \text{ as } n\to\infty.
\]
$(ii)$ Given a parameter space $\Scal \subseteq \Pcal_{\,\bar\ell<\infty}\times\Qcal$, $\hat R$ is \textbf{pointwise consistent} on $\Scal$ iff it is consistent at every state in $\Scal$.\par
$(iii)$ Distances between distributions over texts---data distributions and models alike---are measured by the total variation distance $\dTV(p,p') := \sup_{A\subseteq\X}|p(A)-p'(A)|$, and distances between states $(p,q)$ and $(p',q')$ by the sum $\dTV(p,p') + \dTV(q,q')$---the city-block combination of the two total variation distances. A subset $D$ of a set $\Scal$---of states, of data distributions, or of models---is \textbf{dense} in $\Scal$ iff every point of $\Scal$ has points of $D$ arbitrarily close to it in the distance just given for points of $\Scal$.\par
$(iv)$ $\hat R$ is \textbf{densely inconsistent} on $\Scal$ iff the states in $\Scal$ at which it is not consistent form a dense subset of $\Scal$.
\end{definition}

Pointwise consistency is among the lowest standards often used for assessing inference procedures, demanding neither convergence rates nor uniform convergence, so failing it is a strong negative result. One can lower the bar further. A tradition in causal discovery tolerates failure on negligible sets (Spirtes et al., 2000; Meek, 1995), and one version of it requires consistency except on a nowhere dense set (Lin and Zhang, 2020). Dense inconsistency violates even that very low standard.

\paragraph{Tests.}
A hypothesis-testing problem asks whether a statistical hypothesis $H$ is true, with two possible answers, $\mathtt{yes}$ and $\mathtt{no}$. Let $\Pcal' \subseteq \Pcal$ be the set of distributions under consideration, and let $H \subseteq \Pcal'$. A {\em test} of $H$ is a sequence of maps $\phi_n : \X^n \to \{\mathtt{yes},\mathtt{no}\}$. It is \textbf{consistent at} $p \in \Pcal'$ iff, as $n \to \infty$, $\PP_{p,n}(\phi_n(Y)=\mathtt{yes})$ tends to $1$ if $p \in H$ and to $0$ if $p \notin H$. Pointwise consistency and dense inconsistency on $\Pcal'$ are defined as for estimators.

% =====================================================================
\section{Hardness}\label{sec-hard}

Fix a bound $L > 1$ on expected text length, and let $\Pcal_{\,\bar\ell<L} := \{p \in \Pcal : \bar\ell(p) < L\}$. Let $\Qfull := \{q \in \Qcal : q(x) > 0 \text{ for all } x \in \X\}$ be the full-support models, which are what softmax output layers produce: a softmax makes every next-token probability positive, hence every text probability positive.

\begin{theorem}[Holdout Is No Exception]\label{thm-main}
Let $L > 1$, and let $\hat R$ be any estimator.\\
$(a)$ $\hat R$ is not pointwise consistent on $\Pcal_{\,\bar\ell<\infty}\times\Qcal$, and not even on the smaller space $\Pcal_{\,\bar\ell<L}\times\Qfull$.\\
$(b)$ Worse, $\hat R$ is densely inconsistent on $\Pcal_{\,\bar\ell<L}\times\Qfull$.
\end{theorem}

Part $(b)$ implies $(a)$. The same holds for the per-text risk
\[
R'(p,q) \;:=\; \E_p\big[-\log q(X)/|X|\big],
\]
which averages each text's per-token loss over texts (Appendix~\ref{app-rho-prime}).

One may object that Theorem~\ref{thm-main} is an artifact of the error measure: the holdout average does converge to $\infty$, in the sense of real analysis. But, to reiterate, it never reaches $\infty$, and a finite estimate of an infinite risk misses it by an infinite amount, however large the estimate. One may then propose a bounded measure of error. To adopt one coherently, however, is to give up the log loss for a different loss function under which every risk is finite, so that the straightforward difference applies again. And there are good reasons to keep the log loss---it is, in essence, the only local and proper scoring rule (Appendix~\ref{app-loss})---and hence the cross-entropy risk.

The restriction to $\Pcal_{\,\bar\ell<L}$ is an assumption that no data can check, and so is even the bare existence of the risk. For $c \geq 1$, let $\Pcal_{\,\bar\ell\le c} := \{p \in \Pcal : \bar\ell(p) \leq c\}$.

\begin{proposition}\label{prop-length}
$(a)$ No test of the hypothesis $\Pcal_{\,\bar\ell<\infty}$ is pointwise consistent on $\Pcal$; indeed, every such test is densely inconsistent on $\Pcal$.\\
$(b)$ Let $L > 1$. No test of the hypothesis $\Pcal_{\,\bar\ell<L}$ is pointwise consistent on $\Pcal$; indeed, the distributions at which a given test is inconsistent are dense in $\Pcal_{\,\bar\ell\le L}$.
\end{proposition}

So Theorem~\ref{thm-main} holds even granting an untestable assumption.\footnote{Strictness matters for Proposition~\ref{prop-length}$(b)$, but not for Theorem~\ref{thm-main}. The weak-inequality class $\Pcal_{\,\bar\ell\le L}$ is closed and its complement open, a pair that the characterization of Genin and Kelly (2017) classifies as decidable in the limit (Appendix~\ref{app-proof-length}). Yet Theorem~\ref{thm-main} holds verbatim with $\Pcal_{\,\bar\ell\le L}$ in place of $\Pcal_{\,\bar\ell<L}$, so the impossibility does not rest on the assumption's being untestable.}

The proof of Theorem~\ref{thm-main} has three ingredients. The first shows that an estimator that succeeds must, in effect, decide whether the risk is finite.

\begin{lemma}[Reduction]\label{lem-reduction}
Fix $q \in \Qcal$ and an estimator $\hat R$. Define a test of the hypothesis $H_q := \{p \in \Pcal_{\,\bar\ell<\infty} : R(p,q) < \infty\}$ by letting $\phi_n(y) := \mathtt{yes}$ iff $\hat R_n(y;q) < \infty$. If $\hat R$ is consistent at $(p,q)$, then $\phi$ is consistent at $p$.
\end{lemma}

The reason is simple: an estimate within error $1$ of an infinite risk must itself be $\infty$, and one within error $1$ of a finite risk must be finite (Appendix~\ref{app-prelim}).

The second ingredient consists of two density facts. Just as a laconic person tends to stop talking early, call a full-support predictive distribution $q$ \emph{laconic} if its probability of producing a text drops appreciably as the text grows longer---or, equivalently, if its negative log probability, the log loss, grows appreciably, in this sense: there are constants $A > 0$ and $B$ such that $-\log q(x) \geq A\,|x|^2 + B$ for all texts $x$ longer than some fixed threshold (any fixed threshold would do). This is a tail property: it is a condition only on texts of arbitrarily large length. Equivalently, $q$ produces a text longer than $\ell$ with probability at most $e^{-A'\ell^2+B'}$ for some $A'>0$ and $B'$ (Appendix~\ref{app-steep}). If, say, its probability of continuing at position $t$ is at most $e^{-2t}$, it produces a text longer than $\ell$ with probability at most $e^{-\ell(\ell+1)}$, and hence the log loss grows at least quadratically: $-\log q(x) \geq |x|\,(|x|-1)$ for every text $x$. Transformers cannot be laconic (as we will see in Theorem~\ref{thm-transformer} below), but other architectures can: a recurrent network with ReLU activations, for example, can count the tokens it has produced and lower its logits for continuing in proportion to the count, which makes it laconic. The density facts are these:

\begin{lemma}[Density]\label{lem-density}
$(a)$ Laconic models are dense in $\Qfull$.\\
$(b)$ Let $q$ be laconic and $c > 1$. Then $\{p \in \Pcal_{\,\bar\ell\le c} : R(p,q) < \infty\}$ and its complement, $\{p \in \Pcal_{\,\bar\ell\le c} : R(p,q) = \infty\}$, are both dense in $\Pcal_{\,\bar\ell\le c}$.
\end{lemma}

The mechanism behind part $(b)$ is already instantiated in the example of Section~\ref{sec-setting}. A text of length $\ell$ contributes in proportion to $\ell$ to expected length, but in proportion to $\ell^2$ to a laconic model's expected loss. Moving a sliver of probability onto long texts therefore costs arbitrarily little in total variation and in expected length, yet makes the risk infinite; truncating the tail does the reverse. And a change of size $\epsilon$ in total variation alters the probability of any event concerning $n$ sampled texts by at most $n\epsilon$ (Appendix~\ref{app-prelim}), so a sample can hardly tell the two distributions apart.

The third ingredient shows that when both answers to a question are dense in a closed set of distributions, every test of it fails there, and fails densely.

\begin{lemma}[Vacillation]\label{lem-vacillation}
Let $\Pcal' \subseteq \Pcal$ be closed in total variation, and let a hypothesis $H \subseteq \Pcal'$ be such that $H$ and $\Pcal'\setminus H$ are both dense in $\Pcal'$. Then every test of $H$ is densely inconsistent on $\Pcal'$.
\end{lemma}

Lemma~\ref{lem-vacillation} also follows from Genin and Kelly (2017) with the Baire category theorem (Appendix~\ref{app-vacillation}), but we give a direct forcing argument in the tradition of Gold (1967) and Putnam (1965), to show what goes wrong for any test: whenever a test settles on an answer, nature can move the truth slightly and drive the test back; repeating this forever yields, in the limit, a distribution at which the test never settles---having {\em vacillated} between answers forever, hence the name of the lemma.

With the three lemmas in hand, the impossibility result, Theorem~\ref{thm-main}, follows by chaining them together with some simple topological arguments.

\paragraph{Proof sketch of Theorem~\ref{thm-main}.}
Given a nonempty open set $U \times V$ in $\Pcal_{\,\bar\ell<L}\times\Qfull$, pick a laconic $q \in V$ (Lemma~\ref{lem-density}$(a)$) and some $p^\circ \in U$, and choose $c$ with $\bar\ell(p^\circ) < c < L$. As it turns out, the set $\Pcal_{\,\bar\ell\le c}$ is closed---for a topological reason: expected length is lower semicontinuous in total variation, so it can drop in a limit but cannot jump up (Lemma~\ref{lem-lsc}, Appendix~\ref{app-prelim}). By Lemmas~\ref{lem-density}$(b)$ and~\ref{lem-vacillation}, the test induced by $\hat R$ and $q$ (as in Lemma~\ref{lem-reduction}) is inconsistent at some $p \in \Pcal_{\,\bar\ell\le c}$ arbitrarily close to $p^\circ$, hence in $U$. Then, thanks to Lemma~\ref{lem-reduction}, inconsistency of the test induces inconsistency of the estimator: $\hat R$ is inconsistent at $(p,q) \in U\times V$. This establishes part $(b)$, which implies part $(a)$. So we are done: holdout is no exception, and neither is any other estimator. But note why we make a detour through $\Pcal_{\,\bar\ell\le c}$ with $c < L$, rather than working directly in $\Pcal_{\,\bar\ell<L}$: it is the weak inequality ``$\leq$'' that makes the class closed, and closedness is what makes Lemma~\ref{lem-vacillation} applicable. Indeed, Lemma~\ref{lem-vacillation} can fail for classes that are not closed (Appendix~\ref{app-vacillation}). Proposition~\ref{prop-length} is proved the same way (Appendix~\ref{app-proof-length}).

There are at least two ways to escape the negative result: restrict the class of models, or modify the estimand. We discuss them in turn, in Sections~\ref{sec-models} and~\ref{sec-estimand}.

% =====================================================================
\section{Restricting the Model Class}\label{sec-models}

The problem with laconic models is that their log loss increases too fast---at least quadratically. To restore consistency we need a property of models that caps how fast their log loss can grow with length. Since we will talk a lot about conditional probabilities from now on, note that for a full-support model $q \in \Qfull$ the next-token conditionals $q(v\mid s)$ are defined in every context $s$.\footnote{For $q \in \Qcal$ and a context $s$, let $q(s\cdot)$ be the probability of the texts that begin with $s$; then $q(v \mid s) := q(sv\cdot)/q(s\cdot)$, where $q(s\,\eos\cdot) := q(s\,\eos)$, and $q(v\mid s) := 0$ if $q(s\cdot) = 0$. With this convention the chain rule $-\log q(x) = \sum_t -\log q(x_t\mid x_{<t})$ holds in $[0,\infty]$ for every $q \in \Qcal$.}

\begin{definition}\label{def-floor}
A \textbf{floor} of a model $q \in \Qfull$ is a {\em positive} lower bound on all of its next-token probabilities: a number $\beta > 0$ such that $q(v \mid s) \geq \beta$ for every context $s$ and every token $v$, $\eos$ included. Let $\Qfloor$ be the set of models that have a floor.\end{definition}

A text's probability is a product of conditionals, and a floor $\beta$ bounds every factor:
\begin{align*}
-\log q(x) &= \sum_{t=1}^{|x|} -\log q(x_t\mid x_{<t})\\
&\leq \sum_{t=1}^{|x|}\log(1/\beta) \;=\; |x|\log(1/\beta).
\end{align*}
So log loss grows at most linearly in length, never quadratically as in the example of Section~\ref{sec-setting}. Taking expectations, $R(p,q) \leq \log(1/\beta)$ for every $p \in \Pcal_{\,\bar\ell<\infty}$, and the strong law of large numbers gives the following (proved in Appendix~\ref{app-floor}).

\begin{proposition}\label{prop-floor}
The holdout estimator is pointwise consistent on $\Pcal_{\,\bar\ell<\infty}\times\Qfloor$.
\end{proposition}

Beyond i.i.d.\ sampling, the only assumption on the data distribution is that the expected text length is finite, $\bar\ell(p) < \infty$, which is what it takes for the risk to exist; the bound $L$ is dispensable. But are there models of real interest that have a floor? Yes, as the following two examples show.
   
\paragraph{Transformers have a floor.}
Let $\QT$ be the class of models computed by a fixed-depth Transformer: bounded positional vectors; finitely many layers, each mixing positions only by attention---a convex combination of value vectors---and otherwise applying continuous position-wise maps; and a softmax over logits affine in the last representation (Appendix~\ref{app-transformer}).

\begin{theorem}[Way Out by a Floor Proved from within an Architecture]\label{thm-transformer}
Every model in $\QT$ has a floor. Hence the holdout estimator is pointwise consistent on $\Pcal_{\,\bar\ell<\infty}\times\QT$.
\end{theorem}

\paragraph{Proof sketch of Theorem~\ref{thm-transformer}.}
However long the context, attention {\em only averages} and every other map in a layer is {\em continuous}; from these two facts, a short topological argument shows that internal representations stay in a fixed {\em compact} set (Appendix~\ref{app-transformer}). The logits are then bounded, say within $[-B,B]$, so every token receives probability at least $e^{-2B}/|\Vcal|$, a floor; Proposition~\ref{prop-floor} applies.

The theorem covers the positional schemes in common use (Appendix~\ref{app-transformer}). The proof adapts that of Du et al.\ (2023, Lemma~5.8 and Theorem~5.9), who bound the conditional probability of $\eos$ alone---all that tightness requires. But the same argument bounds the conditional probability of every token, which is what we exploit. Our addition is the bridge from the architectural to the statistical: the observation that this architecture theorem delivers a floor, and that a floor suffices to restore pointwise consistency.

\paragraph{Finite context also gives a floor.}
Say that $q$ has a \emph{context window} $w$ iff $q(v\mid s) = q(v\mid s_w)$ for every context $s$ and token $v$, where $s_w$ is the context formed by the last $w$ tokens of $s$, or $s$ itself if $s$ has fewer than $w$ tokens. Any full-support model $q'$, whatever its internal architecture, can be turned by an external design into a different model with window $w$: the \emph{windowed model} $q'_w$, defined by $q'_w(v\mid s) := q'(v\mid s_w)$, which has window $w$ by the idempotence property $(s_w)_w = s_w$.

\begin{theorem}[Way Out by a Floor through External Design]\label{thm-window}
Every model in $\Qfull$ with a finite context window has a floor. Hence the holdout estimator is pointwise consistent on $\Pcal_{\,\bar\ell<\infty} \times \{q \in \Qfull : q \text{ has a finite context window}\}$.
\end{theorem}

\paragraph{Proof sketch of Theorem~\ref{thm-window}.}
With a window there are only finitely many distinct contexts $s_w$, each with positive conditionals, and the minimum of finitely many positive numbers is positive.

The same argument applies if the window may vary but never exceeds a predetermined finite size. The result is architecture-neutral: the floor follows from features visible in the design---a finite window and a softmax output---with no theorem about long-context behavior.

\paragraph{Routes to a floor.}
To recap, there are two routes to a floor---but definitely not the only two. We present these two because they represent two very different strategies. The finite-context route (Theorem~\ref{thm-window}) secures a floor by {\em external design}, whatever the internals, at the cost of ignoring everything beyond the window. The Transformer route (Theorem~\ref{thm-transformer}) obtains a floor from the architecture's internals, over unboundedly many contexts, and gives up nothing; but we know that it gives a floor only through a {\em proof}. A third route by external design is available: smoothing, which mixes every next-token conditional with the uniform distribution (Appendix~\ref{app-smoothing}).

\paragraph{Why proof or design is needed.}
Could a floor instead be certified by approximate knowledge of $q$---by sampling from it, say, or by estimating it to any precision? No: models with a floor and laconic models are each dense among full-support models, so approximate knowledge can neither certify nor exclude a floor---only proof or design can (Appendix~\ref{app-model-density}).

\paragraph{Floors in implementation.}
One may worry that the negative results are moot in practice: models are computed in floating-point arithmetic, which cannot represent the arbitrarily small probabilities of a laconic tail. We grant that finite precision is also a way out, but that reinforces our point. Distinguish \emph{model development}, where a model is specified over the reals, from \emph{implementation}, where it is computed in floating point. A bounded exponent range makes each implemented next-token probability either zero or bounded below; the holdout estimator is then consistent, since zeros make the risk infinite only detectably (Appendix~\ref{app-truncation}), and the bound below acts as a floor. This paper studies the statistical dividend of model development; the worry points to studying, in addition, the dividend of implementation---which, too, comes into view only when the negative and the positive results are put side by side, to see what turns the one into the other.
% =====================================================================
\section{Modifying the Estimand}\label{sec-estimand}

What if the architecture or design mentioned above becomes too restrictive for developing language models, or nothing about the model under study can be certified---a new architecture, no theorem, no window? In that case, Section~\ref{sec-models} offers nothing. Yet, instead of restricting the parameter space, we can modify the target of estimation: perhaps estimating the original estimand is actually overkill for some important purposes. Consider, then, a threshold $r \in (0,\infty)$ and the \emph{capped risk} $\min\{R(p,q), r\}$: report the risk if it is below $r$, and otherwise report only that it is at least $r$.

\begin{theorem}[Way Out by a Capped Estimand]\label{thm-capped}
For every $r \in (0,\infty)$, the capped holdout estimator $\min\{\hat R^{\hold}_n, r\}$ is consistent for the capped risk $\min\{R(p,q), r\}$ at every state in $\Pcal_{\,\bar\ell<\infty}\times\Qcal$.
\end{theorem}

The cap removes the need to decide finiteness. Once capped, an infinite risk and a risk of $10r$ look alike, so the estimator need only see the holdout average climb past $r$, which, by the strong law of large numbers, it eventually does (Appendix~\ref{app-capped}). Theorem~\ref{thm-capped} holds on the same parameter space as the first claim of Theorem~\ref{thm-main}$(a)$, the largest on which $R$ is defined. So the impossibility consists entirely in the demand to decide finiteness, and the holdout method achieves everything short of that. Moreover, the capped estimand is already enough for an important practice: screening.

\paragraph{Enough for the practice of screening.}
Held-out loss screens the many candidates of pretraining---checkpoints, runs, architectures---and screening has a veto-and-rank structure: candidates above $r$ are vetoed, those below are ranked, and the exact risk of a vetoed candidate never affects the choice. A natural cap is $r = \log|\Vcal|$, the risk of the uniform model, so that any model worse than uniform guessing is vetoed; serious pretrained models score far below it, where the capped estimate is the holdout estimate.

% =====================================================================
\section{Discussion}\label{sec-discussion}

What licenses reading the holdout number as an estimate of a language model's risk? There need to be assumptions about the data: our standing assumption of i.i.d.\ sampling, and finite expected text length, which is what it takes for the risk to exist. But these are not enough unless we also have a way out of the impossibility result (Theorem~\ref{thm-main}): a floor proved from within an architecture, as with Transformers (Theorem~\ref{thm-transformer}); a floor through external design, as with finite context windows (Theorem~\ref{thm-window}); a capped estimand with a more modest objective (Theorem~\ref{thm-capped}); or any escape yet to be found. Data alone never certify the holdout number.

Two engineering choices thereby pay a statistical dividend. Finite context windows and the Transformer (Vaswani et al., 2017) were adopted for computational reasons, and the Transformer's floor is a by-product nobody designed for. Today's Transformers thus collect the dividend, so for them the negative result does not bite---and that is the point: what protects the holdout number is a theorem about the architecture, not anything in the data. An architecture without such a theorem, such as the state-space models discussed below, collects no dividend until one is proved. Smoothing is different: it already had a statistical rationale, avoiding zero probabilities and hence infinite loss on unseen events (Chen and Goodman, 1999). Our results extend that rationale, since avoiding zeros removes only the kind of infinite risk that data can detect (Appendix~\ref{app-truncation}), whereas a floor removes the kind no data can detect. {\em Lesson}: machine learning has an engineering side and a statistical side, and some feats of engineering pay a statistical dividend---one that can be seen only by putting negative and positive results together and identifying what turns the negative into the positive.

%The product parameter space is a general device: whenever an estimand depends on an unknown data distribution and on an object we have built but do not fully understand---a simulator, a numerical code, a trained model---the parameter space should range over both. Data narrow the empirical coordinate, and proofs narrow the mathematical one.

The remaining exits are discussed in Appendix~\ref{app-ways-out}: restricting the data-distribution class by bounding text length, changing the loss function, finer topologies, and decoding-time truncation. One question remains open: model selection. Good risk estimates need not make a good selection rule---an old lesson, as AIC illustrates (Shao, 1997; Yang, 2005)---and close ties add vacillation of their own, so veto-and-rank selection needs its own analysis.

Perhaps the positive results here are not yet positive enough. What about convergence rates, or even uniform consistency? This paper does not aim at such results---not because pointwise consistency is good enough, for it is not, but because the point of this project is to propose a program: to pinpoint the statistical dividends of designs originally adopted for non-statistical reasons, reasons on the engineering side, such as computational cost. We have illustrated the program with examples: pointwise consistency is earned by the Transformer and by finite context windows, through the floors they provide. To pursue the program further, we need to ask: do these designs also earn higher statistical dividends, such as convergence rates or uniform consistency? For the per-text risk $R'$, they already do, through the same floors: uniform consistency over all data distributions (Appendix~\ref{app-rho-prime}). If not for $R$, do other designs earn such higher dividends? And do newer designs, such as state-space models (Gu and Dao, 2023), earn even pointwise consistency? We leave these questions to future work.

\section*{AI Use Statement}

In this work, we used generative AI tools (Claude, by Anthropic) for the following tasks requiring disclosure: literature discovery; brainstorming terminology and possible outlines for the paper; proving or refining some conjectured technical lemmas during the early stages of the project; checking the rigor of proofs and suggesting remedies for proof gaps; drafting some proofs from strategies outlined by the authors; drafting passages in the main text from detailed outlines provided by the authors; drafting and editing text for readability; and acting as a simulated reviewer by posing critical questions for the authors to address.

We did not use generative AI tools for synthetic-data generation, experiments, method implementation, translation, dataset cleaning or reformatting, or qualitative analysis, none of which are applicable to this work.

All AI-assisted work was reviewed by the authors. In particular, all AI-assisted proofs were checked line by line, and all citations were verified against the original sources. We take full responsibility for the final content of this work, including all text, claims, and artifacts produced with the aid of generative AI.

Finally, the generative AI tools raised critical questions about the significance of the main results and suggested possible responses. We found most of these suggested responses unpersuasive but still very valuable---our efforts to articulate why they were unpersuasive helped us develop clearer and stronger explanations of the significance of the main results.

\section*{References}

\begin{description}\setlength{\itemsep}{0.15em}

\item Antos, A., and I. Kontoyiannis (2001). ``Convergence properties of functional estimates for discrete distributions.'' {\em Random Structures \& Algorithms} 19(3--4), 163--193.

\item Arrow, K. J. (1951). {\em Social Choice and Individual Values}. Wiley.

\item Asmussen, S., P. W. Glynn, and H. Thorisson (1992). ``Stationarity detection in the initial transient problem.'' {\em ACM Transactions on Modeling and Computer Simulation} 2(2), 130--157.

\item Bahadur, R. R., and L. J. Savage (1956). ``The nonexistence of certain statistical procedures in nonparametric problems.'' {\em Annals of Mathematical Statistics} 27(4), 1115--1122.

\item Bernardo, J. M. (1979). ``Expected information as expected utility.'' {\em The Annals of Statistics} 7(3), 686--690.

\item Boeken, P., E. Skapinakis, K. Genin, and J. M. Mooij (2026). ``Topological criteria for hypothesis testing with finite-precision measurements.'' arXiv:2601.13946.

\item Brier, G. W. (1950). ``Verification of forecasts expressed in terms of probability.'' {\em Monthly Weather Review} 78(1), 1--3.

\item Casella, G., and R. L. Berger (2002). {\em Statistical Inference}, 2nd ed. Pacific Grove, CA: Duxbury.

\item Chen, S. F., and J. Goodman (1999). ``An empirical study of smoothing techniques for language modeling.'' {\em Computer Speech \& Language} 13(4), 359--394.

\item Dembo, A., and Y. Peres (1994). ``A topological criterion for hypothesis testing.'' {\em The Annals of Statistics} 22(1), 106--117.

\item Diaconis, P. (1988). ``Bayesian numerical analysis.'' In S. S. Gupta and J. O. Berger (eds.), {\em Statistical Decision Theory and Related Topics IV}, Vol.~1, 163--175. Springer.

\item Donoho, D. L. (1988). ``One-sided inference about functionals of a density.'' {\em The Annals of Statistics} 16(4), 1390--1420.

\item Du, L., L. Torroba Hennigen, T. Pimentel, C. Meister, J. Eisner, and R. Cotterell (2023). ``A measure-theoretic characterization of tight language models.'' In {\em Proceedings of the 61st Annual Meeting of the Association for Computational Linguistics (Volume 1: Long Papers)}, 9744--9770.

\item Durrett, R. (2019). {\em Probability: Theory and Examples}, 5th edition. Cambridge University Press.

\item Ermakov, M. (2017). ``On consistent hypothesis testing.'' {\em Journal of Mathematical Sciences} 225(5), 751--769.

\item Genin, K., and K. T. Kelly (2017). ``The topology of statistical verifiability.'' In {\em Proceedings of the Sixteenth Conference on Theoretical Aspects of Rationality and Knowledge (TARK 2017)}, Electronic Proceedings in Theoretical Computer Science 251, 236--250.

\item Genin, K. (2018). {\em The Topology of Statistical Inquiry}. Ph.D. dissertation, Carnegie Mellon University.

\item Gneiting, T., and A. E. Raftery (2007). ``Strictly proper scoring rules, prediction, and estimation.'' {\em Journal of the American Statistical Association} 102(477), 359--378.

\item Gold, E. M. (1967). ``Language identification in the limit.'' {\em Information and Control} 10(5), 447--474.

\item Gu, A., and T. Dao (2023). ``Mamba: Linear-time sequence modeling with selective state spaces.'' arXiv:2312.00752.

\item Hennig, P., M. A. Osborne, and M. Girolami (2015). ``Probabilistic numerics and uncertainty in computations.'' {\em Proceedings of the Royal Society A} 471(2179), 20150142.

\item Hoeffding, W. (1963). ``Probability inequalities for sums of bounded random variables.'' {\em Journal of the American Statistical Association} 58(301), 13--30.

\item Hoffmann, J., et al.\ (2022). ``Training compute-optimal large language models.'' In {\em Advances in Neural Information Processing Systems} 35.

\item Holtzman, A., J. Buys, L. Du, M. Forbes, and Y. Choi (2020). ``The curious case of neural text degeneration.'' In {\em International Conference on Learning Representations}.

\item Kaplan, J., et al.\ (2020). ``Scaling laws for neural language models.'' arXiv:2001.08361.

\item Kelly, K. T. (1996). {\em The Logic of Reliable Inquiry}. Oxford University Press.

\item Kennedy, M. C., and A. O'Hagan (2001). ``Bayesian calibration of computer models.'' {\em Journal of the Royal Statistical Society: Series B} 63(3), 425--464.

\item Levin, D. A., and Y. Peres (2017). {\em Markov Chains and Mixing Times}, 2nd edition. American Mathematical Society.

\item Lin, H., and J. Zhang (2020). ``On learning causal structures from non-experimental data without any faithfulness assumption.'' In {\em Proceedings of the 31st International Conference on Algorithmic Learning Theory}, PMLR 117, 554--582.

\item List, C. (2022). ``Social Choice Theory.'' In E. N. Zalta and U. Nodelman (eds.), {\em The Stanford Encyclopedia of Philosophy}, Winter 2022 edition. \url{https://plato.stanford.edu/archives/win2022/entries/social-choice/}.

\item McCutcheon, R. G. (2019). ``In favor of logarithmic scoring.'' {\em Philosophy of Science} 86(2), 286--303.

\item Meek, C. (1995). ``Strong completeness and faithfulness in Bayesian networks.'' In {\em Proceedings of the Eleventh Conference on Uncertainty in Artificial Intelligence}, 411--418.

\item Mossel, E., and M. I. Ohannessian (2019). ``On the impossibility of learning the missing mass.'' {\em Entropy} 21(1), 28.

\item Press, O., N. A. Smith, and M. Lewis (2022). ``Train short, test long: Attention with linear biases enables input length extrapolation.'' In {\em International Conference on Learning Representations}.

\item Putnam, H. (1965). ``Trial and error predicates and the solution to a problem of Mostowski.'' {\em Journal of Symbolic Logic} 30(1), 49--57.

\item Shah, R. D., and J. Peters (2020). ``The hardness of conditional independence testing and the generalised covariance measure.'' {\em The Annals of Statistics} 48(3), 1514--1538.

\item Shalev-Shwartz, S., and S. Ben-David (2014). {\em Understanding Machine Learning: From Theory to Algorithms}. Cambridge University Press.

\item Shao, J. (1997). ``An asymptotic theory for linear model selection.'' {\em Statistica Sinica} 7(2), 221--264. With discussion.

\item Shazeer, N., A. Mirhoseini, K. Maziarz, A. Davis, Q. Le, G. Hinton, and J. Dean (2017). ``Outrageously large neural networks: The sparsely-gated mixture-of-experts layer.'' In {\em International Conference on Learning Representations}.

\item Shuford, E. H., A. Albert, and H. E. Massengill (1966). ``Admissible probability measurement procedures.'' {\em Psychometrika} 31(2), 125--145.

\item Spirtes, P., C. Glymour, and R. Scheines (2000). {\em Causation, Prediction, and Search}, 2nd edition. MIT Press.

\item Su, J., M. Ahmed, Y. Lu, S. Pan, W. Bo, and Y. Liu (2024). ``RoFormer: Enhanced transformer with rotary position embedding.'' {\em Neurocomputing} 568, 127063.

\item Traub, J. F., G. W. Wasilkowski, and H. Wo\'zniakowski (1988). {\em Information-Based Complexity}. Academic Press.

\item Vaswani, A., N. Shazeer, N. Parmar, J. Uszkoreit, L. Jones, A. N. Gomez, {\L}. Kaiser, and I. Polosukhin (2017). ``Attention is all you need.'' In {\em Advances in Neural Information Processing Systems} 30.

\item Wyner, A. J., and D. Foster (2003). ``On the lower limits of entropy estimation.'' Manuscript, available at \url{https://deanfoster.net/research/low_limits_entropy.pdf}.

\item Yang, Y. (2005). ``Can the strengths of AIC and BIC be shared? A conflict between model identification and regression estimation.'' {\em Biometrika} 92(4), 937--950.

\end{description}

% =====================================================================
\appendix

% =====================================================================
\section{Details of the Setting}\label{app-setting}

This appendix supplements the setting of Section~\ref{sec-setting}. Appendix~\ref{app-primitives} justifies the per-token risk from the autoregressive point of view, Appendix~\ref{app-error} checks that the error measure is a metric, and Appendix~\ref{app-erm} explains how a context window can be imposed after training. None of this is used in the proofs that follow, so a reader who wishes to focus on the proofs of the main results can skip directly to Appendix~\ref{app-proofs-neg}.

\subsection{The Per-Token Risk from Autoregressive Primitives}\label{app-primitives}

Section~\ref{sec-setting} defines the per-token risk as a ratio of expectations: expected loss per text over expected length per text. Practitioners often picture it differently, as the expected loss at a single next-token prediction. This subsection makes that second picture precise. It takes the autoregressive conditionals as primitive, defines the per-token risk directly as an expected next-token loss, and shows that the direct definition makes sense exactly when the expected text length $\bar\ell(p) = \E_p|X|$ is finite---and then agrees with the ratio.

A \emph{context} is an element of $\Vcal_0^*$, the empty string $\varnothing$ included. Throughout this subsection, $s$ ranges over contexts and $z$ over texts. The \emph{reach probability} $\pi(s)$ of a context $s$ is the probability that a text drawn from $p$ begins with $s$: the probability that the generating process produces $s$ and has not yet halted. So $\pi(\varnothing) = 1$. Where $\pi(s) > 0$, let $p(\cdot\mid s)$ be the next-token conditional of $p$ at $s$; then $\pi(sa) = \pi(s)\,p(a\mid s)$ for every ordinary token $a$, and
\[
\pi(s) \;=\; \prod_{j=1}^{k} p(s_j\mid s_{<j}) \qquad \text{for } s = s_1\cdots s_k.
\]
The natural sample space for next-token prediction is the set $\Vcal_0^* \times \Vcal$ of \emph{prediction events} $(s,v)$: the event that context $s$ is reached and token $v$ is emitted next. A distribution $\nu$ on prediction events, with marginal $\nu_S$ on contexts, represents the process iff it satisfies two conditions:
\begin{itemize}\setlength{\itemsep}{0pt}
\item \emph{kernel agreement}: $\nu(s,v) = \nu_S(s)\,p(v\mid s)$ wherever $\pi(s) > 0$---given the context, the next token is distributed as $p$ prescribes;
\item \emph{prefix consistency}: $\nu_S(sa) = \nu(s,a)$ for every context $s$ and ordinary token $a$---the context $sa$ is reached exactly when $s$ is reached and $a$ is emitted there.
\end{itemize}

\begin{applemma}\label{lem-nu}
A distribution $\nu$ on $\Vcal_0^*\times\Vcal$ satisfying kernel agreement and prefix consistency exists iff $\bar\ell(p) < \infty$. It is then unique, and $\nu(s,v) = \pi(s)\,p(v\mid s)/\bar\ell(p)$, read as $0$ when $\pi(s) = 0$.
\end{applemma}

\begin{proof}
In outline, the proof has two steps. First, the two conditions force $\nu$ to weight each context in proportion to its reach probability. Second, the reach probabilities sum to the expected length. So $\nu$ can be normalized exactly when the expected length is finite. We take the two steps in turn, and then check that the resulting $\nu$ exists.

\emph{Proportionality.} We show $\nu_S(s) = \nu_S(\varnothing)\,\pi(s)$ by induction on the length of $s$; the base case, $s = \varnothing$, is trivial. If $\pi(s) > 0$, the two conditions and the induction hypothesis give
\[
\nu_S(sa) \;=\; \nu(s,a) \;=\; \nu_S(s)\,p(a\mid s) \;=\; \nu_S(\varnothing)\,\pi(s)\,p(a\mid s) \;=\; \nu_S(\varnothing)\,\pi(sa).
\]
If $\pi(s) = 0$, then $\nu_S(s) = 0$ by the induction hypothesis, so $\nu_S(sa) = \nu(s,a) \leq \nu_S(s) = 0 = \nu_S(\varnothing)\,\pi(sa)$.

\emph{Normalization.} The contexts of length $k$ are the ways of producing $k$ ordinary tokens without halting, so $\sum_{s\in\Vcal_0^k}\pi(s) = \PP_p(|X| > k)$. Summing over $k \geq 0$, by the tail-sum formula for the expectation of a positive integer-valued variable,
\[
\sum_s \pi(s) \;=\; \sum_{k\geq0}\PP_p(|X|>k) \;=\; \E_p|X| \;=\; \bar\ell(p).
\]
Since $\nu_S$ must have total mass $1 = \nu_S(\varnothing)\sum_s \pi(s)$, normalization is possible iff $\bar\ell(p) < \infty$, and it then forces $\nu_S(\varnothing) = 1/\bar\ell(p)$. Kernel agreement gives the displayed formula for $\nu$ where $\pi(s) > 0$; where $\pi(s) = 0$, $\nu(s,v) \leq \nu_S(s) = 0$.

\emph{Existence.} Conversely, if $\bar\ell(p) < \infty$, the formula defines a distribution, since $\sum_{s,v}\pi(s)\,p(v\mid s) = \sum_s \pi(s) = \bar\ell(p)$; its marginal on contexts is $\pi/\bar\ell(p)$, so kernel agreement holds, and prefix consistency holds because $\pi(sa) = \pi(s)\,p(a\mid s)$.
\end{proof}

The condition has content: almost-sure termination does not secure it. With $\Vcal_0 = \{\mathtt{a}\}$ and $p(\mathtt{a} \mid \mathtt{a}^{k-1}) = k/(k+1)$, we have
\[
\PP_p(|X|>k) \;=\; \prod_{j=1}^{k}\frac{j}{j+1} \;=\; \frac{1}{k+1} \;\to\; 0,
\]
so the process halts almost surely and $p$ is a distribution on $\X$; yet $\bar\ell(p) = \sum_{k\geq0} 1/(k+1) = \infty$.

\begin{appproposition}\label{prop-primitives}
If $\bar\ell(p) < \infty$ and $\nu$ is as in Lemma~\ref{lem-nu}, then for every $q \in \Qcal$,
\[
\E_{(S,V)\sim\nu}\big[-\log q(V\mid S)\big] \;=\; R(p,q),
\]
the identity holding in $[0,\infty]$, with the convention $q(v\mid s) := 0$ when $q$ gives probability zero to the texts beginning with $s$.
\end{appproposition}

\begin{proof}
The idea is, roughly, that the loss of a text is a sum of next-token losses, and each next-token loss can be credited to the prediction event at which its token is emitted. Regrouping all losses by prediction event turns the expected loss per text into $\bar\ell(p)$ times the expected loss per prediction event. Details follow.

All terms are nonnegative, so every rearrangement below is licensed by Tonelli's theorem. By the chain rule, the loss of a text $z$ is $\sum_{t\leq|z|} -\log q(z_t\mid z_{<t})$, and the term indexed by position $t$ belongs to the prediction event $(z_{<t},z_t)$. The texts that pass through context $s$ and emit $v$ there have total probability $\pi(s)\,p(v\mid s)$, read as $0$ when $\pi(s) = 0$. Hence
\begin{align*}
\E_p[-\log q(X)] &= \sum_z p(z)\sum_{t\leq|z|} -\log q(z_t\mid z_{<t})\\
&= \sum_{(s,v)} \pi(s)\,p(v\mid s)\,\big(-\log q(v\mid s)\big)\\
&= \bar\ell(p)\;\E_{\nu}\big[-\log q(V\mid S)\big],
\end{align*}
and dividing by $\bar\ell(p)$ gives the claim.
\end{proof}

So $R$ is the expected next-token loss at a token drawn uniformly from a large pooled corpus: the risk that the next-token training objective targets.

\subsection{\texorpdfstring{$\boldsymbol{\Error(u,v)}$}{Error(u,v)} Is an Extended Metric}\label{app-error}

Section~\ref{sec-setting} measures error by the straightforward difference, which may be infinite. We check that it is a metric that is allowed to take the value $\infty$, and note the topology it induces. $\Error$ is symmetric and vanishes exactly on the diagonal. For the triangle inequality $\Error(u,w) \leq \Error(u,v) + \Error(v,w)$, there are three cases:
\begin{itemize}\setlength{\itemsep}{0pt}
\item if $u, w < \infty$, the right side is $\infty$ when $v = \infty$, and $|u-v|+|v-w| \geq |u-w|$ otherwise;
\item if exactly one of $u$ and $w$ is infinite, one of the two terms on the right compares a finite with an infinite value, and so equals $\infty$;
\item if $u = w = \infty$, the left side is $0$.
\end{itemize}
In the topology that $\Error$ induces, $\{\infty\}$ is an isolated point: no finite value is close to $\infty$. This is what makes a consistent estimator decide finiteness. In the order topology of $[0,\infty]$, by contrast, $\infty$ is a limit of large finite values; that topology corresponds to a bounded error, equivalent to the capped risk of Section~\ref{sec-estimand} (Appendix~\ref{app-capped}).

\subsection{Training by Empirical Risk Minimization, and Windowing after Training}\label{app-erm}

Nothing in our results depends on how $q$ was trained, but one point bears on the finite-context route of Section~\ref{sec-models}: whether a window can be imposed on a model after it has been trained without departing from the training objective. Empirical risk minimization under log loss over a family $\{q_\theta\}$ is maximum likelihood, and by the chain rule it is the familiar next-token objective; averaging over training tokens rather than texts rescales the objective by a constant independent of $\theta$ and leaves the minimizers unchanged.

Many positional schemes now in use---among them sinusoidal encodings, RoPE, and ALiBi---define the model at every context length and impose no window of their own. A window $w$ can then be imposed after training, in the notation of Section~\ref{sec-models}:
\[
q_{\theta,w}(v\mid s) \;:=\; q_\theta(v \mid s_w), \qquad s_w := \text{the last } w \text{ tokens of } s \text{ (or } s \text{ itself, if shorter)}.
\]
Whether this departs from empirical risk minimization depends on how $w$ compares with the training texts. If some training text is longer than $w$, truncation discards context that the training objective was scored on, and the windowed model need not be a minimizer. But if $w$ is at least the length of every training text, then $q_{\theta,w}$ and $q_\theta$ agree on every training example for every $\theta$; the windowed and unwindowed families have identical empirical risks, and windowing the minimizer yields a minimizer. In that usual case, windowing after training costs nothing in empirical risk and secures a floor (Theorem~\ref{thm-window}), so the risk of the deployed model can be consistently estimated on held-out data.

% =====================================================================
\section{Proofs for Section~\ref{sec-hard}}\label{app-proofs-neg}

This appendix proves the results of Section~\ref{sec-hard}---our negative results. Appendix~\ref{app-prelim} collects tools and proves Lemma~\ref{lem-reduction}, Appendices~\ref{app-steep}--\ref{app-vacillation} prove the other lemmas, culminating in the proof of Theorem~\ref{thm-main} in Appendix~\ref{app-proof-main}, and Appendix~\ref{app-proof-length} proves Proposition~\ref{prop-length}.

\subsection{Preliminaries}\label{app-prelim}

\paragraph{Enumeration.}
Fix an enumeration $\X = \{x_1,x_2,\ldots\}$ ordered by length, so that $x_1 = (\eos)$ is the unique text of length $1$. For distributions $p$ and $q$, write $p_k := p(x_k)$ and $q_k := q(x_k)$, and write $\ell_k := |x_k|$; in this notation, the expected text length is $\bar\ell(p) = \sum_k p_k\ell_k$. Only finitely many texts have any given length; $\ell_k \to \infty$ as $k \to \infty$; and every finite set of texts lies in an initial segment $\{x_1,\ldots,x_N\}$ of the enumeration.

\paragraph{Counting.}
The following fact guarantees that tails shaped like $e^{-\ell^2}$ carry only finite total mass, so that they can be scaled to carry any prescribed probability mass; it is used to construct laconic models (Lemma~\ref{lem-steep-dense}) and models that are laconic of degree $3$ (Appendix~\ref{app-rho-prime}). A text of length $\ell$ consists of $\ell-1$ ordinary tokens followed by $\eos$, so there are exactly $|\Vcal_0|^{\ell-1}$ texts of length $\ell$. Grouping the texts by length,
\begin{equation}\label{eq-counting}
\sum_{k\geq1} e^{-\ell_k^2} \;=\; \sum_{\ell\geq1}\ \sum_{k:\,\ell_k=\ell} e^{-\ell^2} \;=\; \sum_{\ell\geq1} |\Vcal_0|^{\ell-1}e^{-\ell^2} \;<\; \infty,
\end{equation}
where the last series converges by the ratio test: the ratio of its consecutive terms is $|\Vcal_0|\,e^{-(2\ell+1)}$, which tends to $0$. The same holds with $\ell^3$ in place of $\ell^2$. A Gaussian factor defeats a geometric one, which is why nothing below depends on the size of the vocabulary.

\paragraph{Bertrand series.}
Every series in our perturbation arguments (Lemma~\ref{lem-infinite}, Proposition~\ref{prop-length}, and Appendix~\ref{app-rho-prime}) takes the following form, known as the {\em Bertrand series}, so a single classical criterion settles whether it converges. For every integer $n_0 \geq 2$ and all real numbers $\alpha$ and $\beta$,
\begin{equation}\label{eq-bertrand}
\sum_{n=n_0}^{\infty}\frac{1}{n^\alpha(\log n)^\beta}<\infty
\quad\Longleftrightarrow\quad
\alpha>1 \;\text{ or }\; (\alpha=1 \;\text{and}\; \beta>1).
\end{equation}
This classical criterion follows from the comparison test, against $\sum_n n^{-\alpha'}$ for a suitable $\alpha'$ strictly between $\alpha$ and $1$, when $\alpha \neq 1$; and from the integral test when $\alpha = 1$, since the substitution $u = \log x$ turns $\int dx/\big(x(\log x)^\beta\big)$ into $\int du/u^\beta$.

\paragraph{Total variation.}
The total variation distance is half the sum, over all texts, of the absolute differences in probability mass:
\[
\dTV(p,p') \;=\; \frac12\sum_k \big|p_k - p'_k\big|,
\]
half the distance between $p$ and $p'$ in $\ell^1(\X)$---the absolutely summable functions on $\X$. With this distance, $\Pcal$ is a complete metric space, because $\ell^1(\X)$ is complete and $\Pcal$ is closed in it. On a countable $\X$, convergence in total variation is equivalent to coordinatewise convergence of mass functions (by Scheff\'e's lemma), so the total variation topology coincides with the weak topology. We write $B_\epsilon(p)$ for the open ball of radius $\epsilon$ about $p$.

\begin{applemma}[Lower semicontinuity]\label{lem-lsc}
Let $a_k \in [0,\infty)$ for every $k$, and let $\Psi(p) := \sum_k p_k a_k$. Then $\Psi$ is lower semicontinuous, in the sense that $\liminf_j \Psi(p^{(j)}) \geq \Psi(p)$ whenever $p^{(1)}, p^{(2)}, \ldots$ is a sequence of distributions with $\dTV(p^{(j)},p) \to 0$. In particular, expected length $\bar\ell$ is lower semicontinuous: $\liminf_j \bar\ell(p^{(j)}) \geq \bar\ell(p)$ whenever $p^{(j)}$ converges to $p$ in total variation---expected length can drop in the limit, as probability mass escapes to ever longer texts, but it cannot jump up. Hence $\Pcal_{\,\bar\ell\le c} = \{p : \bar\ell(p) \leq c\}$ is closed for every $c$.
\end{applemma}

\begin{proof}
The idea is, roughly, that each finite part of the sum converges, while the rest of the sum, being nonnegative, can only add to the limit. Details follow.

For each $k$, $|p^{(j)}_k - p_k| \leq \dTV(p^{(j)},p) \to 0$. Fix $N$. All terms being nonnegative, $\Psi(p^{(j)}) \geq \sum_{k\leq N} p^{(j)}_k a_k$, and the right side, a finite sum, converges to $\sum_{k\leq N} p_k a_k$ as $j \to \infty$. Hence $\liminf_j \Psi(p^{(j)}) \geq \sum_{k\leq N} p_k a_k$ for every $N$, and letting $N\to\infty$ gives $\liminf_j \Psi(p^{(j)}) \geq \Psi(p)$. Taking $a_k := \ell_k$ gives $\Psi = \bar\ell$. Finally, if $p^{(j)} \in \Pcal_{\,\bar\ell\le c}$ for every $j$ and $p^{(j)} \to p$, then $\bar\ell(p) \leq \liminf_j \bar\ell(p^{(j)}) \leq c$, so $p \in \Pcal_{\,\bar\ell\le c}$.
\end{proof}

\begin{applemma}[Continuity]\label{lem-continuity}
For every $n$, every event $E \subseteq \X^n$, and all $p,p' \in \Pcal$, $|\PP_{p,n}(E) - \PP_{p',n}(E)| \leq n\,\dTV(p,p')$.
\end{applemma}

\begin{proof}
It helps to picture the two samples as generated together, draw by draw, by a mechanism that makes each pair of draws agree except with probability $\dTV(p,p')$. The two samples can then differ only if some pair of draws does, which happens with probability at most $n$ times as large. The proof below makes this picture precise.

Let $\lambda := \dTV(p,p')$, and take a maximal coupling $(X,X')$ of $p$ and $p'$: a pair of random texts with $X \sim p$, $X' \sim p'$, and $\PP(X \neq X') = \lambda$ (Levin and Peres, 2017, Section~4.2). Let $(X^{(1)},X'^{(1)}),\ldots,(X^{(n)},X'^{(n)})$ be independent copies of this pair. Then $Y := (X^{(1)},\ldots,X^{(n)})$ has law $\PP_{p,n}$ and $Y' := (X'^{(1)},\ldots,X'^{(n)})$ has law $\PP_{p',n}$, so, by the union bound,
\begin{align*}
\big|\PP_{p,n}(E) - \PP_{p',n}(E)\big| &= \big|\PP(Y \in E) - \PP(Y' \in E)\big|\\
&\leq \PP(Y \neq Y') \;\leq\; \sum_{j=1}^n \PP\big(X^{(j)} \neq X'^{(j)}\big) \;=\; n\lambda. \qedhere
\end{align*}
\end{proof}

\begin{appremark}[Randomization gains nothing]
A randomized estimator or test uses, in addition to the sample, an independent seed $U\sim\mu$. Its probability of any verdict is $\int \PP_{p,n}(E_u)\,\mu(du)$, which is $n$-Lipschitz in $p$ by Lemma~\ref{lem-continuity}, uniformly in $u$. Every argument below uses only this property, so all results apply verbatim to randomized procedures.
\end{appremark}

\begin{proof}[Proof of Lemma~\ref{lem-reduction}]
Take $\epsilon = 1$ in Definition~\ref{def-main}$(i)$. If $R(p,q) < \infty$, the event $\Error < 1$ forces $\hat R_n < \infty$, so $\PP_{p,n}(\phi_n = \mathtt{yes}) \to 1$. If $R(p,q) = \infty$, it forces $\hat R_n = \infty$, so $\PP_{p,n}(\phi_n = \mathtt{yes}) \to 0$.
\end{proof}

\subsection{Laconic Models and Two Perturbations}\label{app-steep}

Laconic models were defined in Section~\ref{sec-hard} by a tail property: the bound is required only for texts longer than some fixed threshold. Since the enumeration is ordered by length, the texts longer than any given threshold form a tail $\{x_k : k \geq N\}$ of the enumeration, and every such tail contains all texts longer than $\ell_N$. So a model $q \in \Qfull$ is laconic iff there are constants $A > 0$ and $B$ and an index $N$ such that
\begin{equation}\label{eq-steep}
-\log q_k \;\geq\; A\,\ell_k^2 + B \qquad \text{for all } k \geq N.
\end{equation}
Since $-\log q_k \geq 0$, the finitely many texts before the tail can be absorbed into the constant: \eqref{eq-steep} then holds for all $k$, with $B$ replaced by $\min\{B,\,-A\,\ell_N^2\}$. This is why any fixed threshold would do. We use whichever form is convenient. More generally, say that $q$ is \emph{laconic of degree $d$} if $-\log q_k \geq A\,\ell_k^d + B$ on a tail, for some $A > 0$ and $B$; laconic models are those of degree $2$, and Appendix~\ref{app-rho-prime} needs degree $3$, an even more laconic kind.

The name is justified by the following characterization, which describes laconic models by what they generate: they grow ever more eager to stop.

\begin{applemma}[Laconic models stop early]\label{lem-laconic}
A model $q \in \Qfull$ is laconic iff there are constants $A' > 0$ and $B'$ such that, for every $\ell \geq 0$,
\[
\PP_{X\sim q}\big(|X| > \ell\big) \;\leq\; e^{-A'\ell^2 + B'}.
\]
\end{applemma}

\begin{proof}
The idea is, roughly, that the probability of all texts of a given length is at most the number of such texts, $|\Vcal_0|^{\ell-1}$, times the largest of their probabilities, and at least the largest of their probabilities; and the number of texts grows only geometrically in length, which a Gaussian factor absorbs. Details follow.

($\Rightarrow$) Let $A$, $B$, and $N$ be as in~\eqref{eq-steep}. Every text longer than $x_N$ has index above $N$, so for every length $m > \ell_N$,
\[
\PP_{X\sim q}\big(|X| = m\big) \;\leq\; |\Vcal_0|^{m-1}\,e^{-Am^2 - B} \;\leq\; e^{-(A/2)\,m^2 - B},
\]
where the last inequality holds once $(A/2)\,m^2 \geq m\log|\Vcal_0|$, that is, for all $m \geq m_0 := \max\{\ell_N + 1,\ 2\log|\Vcal_0|/A\}$. Hence, for $\ell \geq m_0$, writing $m = \ell + j$ and using $m^2 \geq \ell^2 + j$,
\[
\PP_{X\sim q}\big(|X| > \ell\big) \;\leq\; e^{-B}\sum_{j\geq1} e^{-(A/2)(\ell^2 + j)} \;=\; e^{-(A/2)\,\ell^2}\cdot e^{-B}\sum_{j\geq1}e^{-(A/2)\,j}.
\]
The last sum is finite, so the claim holds for $\ell \geq m_0$ with $A' := A/2$ and a suitable $B'$; enlarging $B'$ covers the finitely many $\ell < m_0$, for which the probability is at most $1$.

($\Leftarrow$) A text $x$ of length $\ell \geq 1$ is one of the texts longer than $\ell - 1$, so $q(x) \leq e^{-A'(\ell-1)^2 + B'}$. Since $(\ell-1)^2 \geq \ell^2/2 - 1$ for every $\ell$, this gives $-\log q(x) \geq (A'/2)\,\ell^2 - A' - B'$ for every text $x$, so $q$ is laconic, with $A := A'/2$ and $B := -A' - B'$.
\end{proof}

\begin{applemma}[Laconic models are dense; Lemma~\ref{lem-density}$(a)$]\label{lem-steep-dense}
For every $q^\circ \in \Qfull$ and $\epsilon > 0$ there is a laconic $q$ with $\dTV(q^\circ,q) < \epsilon$. Hence laconic models are dense in $\Qfull$, and also in $\Qcal$. No laconic model has a floor.
\end{applemma}

\begin{proof}
The idea is, roughly, to keep $q^\circ$ on an initial segment of the enumeration and to replace its tail by a Gaussian tail of the same total mass. The change is small, because a tail far enough out carries little mass; and the new tail makes the model laconic, because a Gaussian tail inflicts a log loss that grows quadratically in length. Details follow.

Since $\sum_k q^\circ_k = 1$, we can choose $N$ such that the tail mass $t := \sum_{k\geq N} q^\circ_k$ is less than $\epsilon$; and $t > 0$, since $q^\circ$ has full support. Let $T := \sum_{k\geq N} e^{-\ell_k^2}$, which is positive and, by the counting fact~\eqref{eq-counting}, finite, and define
\[
q_k \;:=\; \begin{cases} q^\circ_k & \text{if } k < N,\\[2pt] \dfrac{t}{T}\,e^{-\ell_k^2} & \text{if } k \geq N.\end{cases}
\]
Then $\sum_k q_k = (1-t) + t = 1$ and every $q_k$ is positive, so $q \in \Qfull$. Only the tail has changed, and both tails carry total mass $t$, so
\[
\dTV(q^\circ,q) \;=\; \frac12\sum_{k\geq N}\big|q^\circ_k - q_k\big| \;\leq\; \frac12\,(t + t) \;=\; t \;<\; \epsilon.
\]
For every $k \geq N$, $-\log q_k = \ell_k^2 - \log(t/T)$, which is~\eqref{eq-steep} with $A := 1$ and $B := -\log(t/T)$; so $q$ is laconic.

Laconic models are therefore dense in $\Qfull$. They are dense in $\Qcal$ as well, because $\Qfull$ is dense in $\Qcal$: for $q^\circ \in \Qcal$, any $u \in \Qfull$, and $\delta \in (0,1]$, the mixture $(1-\delta)\,q^\circ + \delta\, u$ has full support and lies within $\delta$ of $q^\circ$.

Finally, if a model $q$ had a floor $\beta$, the chain rule would give $-\log q_k \leq \ell_k\log(1/\beta)$ for every $k$ (Section~\ref{sec-models}). Since $\ell_k \to \infty$, this is incompatible with $-\log q_k \geq A\,\ell_k^2 + B$ for all $k \geq N$; so no laconic model has a floor.
\end{proof}

\begin{applemma}[Finite risk, for free]\label{lem-finite}
Let $q \in \Qfull$. For every $p \in \Pcal$ and $\epsilon > 0$ there is a finitely supported $p'$ with $\dTV(p,p') < \epsilon$, $\bar\ell(p') \leq \bar\ell(p)$, and $\E_{p'}[-\log q(X)] < \infty$.
\end{applemma}

\begin{proof}
Informally, we cut off a tail of $p$ and pile its mass onto the shortest text. That leaves only finitely many texts to incur loss, and makes no text longer. Now the details.

Choose $N \geq 1$ such that the tail mass $t := \sum_{k>N} p_k$ is less than $\epsilon$, and let $p'$ be obtained from $p$ by moving the tail mass to the shortest text $x_1$:
\[
p'_1 := p_1 + t, \qquad p'_k := p_k \ \ (2 \leq k \leq N), \qquad p'_k := 0 \ \ (k > N).
\]
Then $\dTV(p,p') = \frac12\big(t + \sum_{k>N} p_k\big) = t < \epsilon$. Expected length does not increase, because the moved mass now sits on a text of length $\ell_1 = 1$, whereas every text has length at least $1$:
\[
\bar\ell(p') \;=\; \sum_{k\leq N} p_k\ell_k + t \;=\; \sum_{k\leq N} p_k\ell_k + \sum_{k>N} p_k \;\leq\; \sum_{k\leq N} p_k\ell_k + \sum_{k>N} p_k\ell_k \;=\; \bar\ell(p).
\]
Finally, $\E_{p'}[-\log q(X)] = \sum_{k\leq N} p'_k\,(-\log q_k)$ is a finite sum of finite terms, since $q$ has full support.
\end{proof}

\begin{applemma}[Infinite risk, cheaply]\label{lem-infinite}
Let $q$ be laconic. For every $p \in \Pcal_{\,\bar\ell<\infty}$ and all $\epsilon,\eta > 0$ there is $p''$ with $\dTV(p,p'') < \epsilon$, $\bar\ell(p'') \leq \bar\ell(p) + \eta$, and $\E_{p''}[-\log q(X)] = \infty$.
\end{applemma}

\begin{proof}
To obtain $p''$, the idea is to perturb $p$ by moving it slightly toward a fixed distribution $w$ that has two key features: its expected length is finite, and it inflicts an infinite expected loss on any laconic model. The result of such a perturbation, $p''$, will have an expected length that is still finite and at most slightly increased, and it will inherit, from a tail of $w$, the ability to inflict an infinite expected loss on any laconic model. So the proof has two steps, which we take in turn: first we construct $w$, then we perturb $p$ toward it.

\emph{The distribution $w$.} Let
\begin{equation}\label{eq-weights}
w_k \;:=\; \begin{cases} \dfrac{1}{Z\,|\Vcal_0|^{\ell_k-1}\,\ell_k^2\,(\log \ell_k)^2} & \text{if } \ell_k \geq 2,\\[8pt] 0 & \text{if } \ell_k = 1,\end{cases}
\end{equation}
where the normalizing constant $Z := \sum_{\ell\geq2} 1/\big(\ell^2(\log \ell)^2\big)$ is finite by Bertrand's criterion~\eqref{eq-bertrand} with $\alpha = 2$ and $\beta = 2$. The $|\Vcal_0|^{\ell-1}$ texts of length $\ell \geq 2$ jointly receive probability mass $1/\big(Z\,\ell^2(\log \ell)^2\big)$, so $\sum_k w_k = 1$. Thus, $w$ is indeed a distribution on $\X$. Its expected length is finite, by Bertrand's criterion~\eqref{eq-bertrand} with $\alpha = 1$ and $\beta = 2$:
\[
\bar\ell(w) \;=\; \sum_k w_k\ell_k \;=\; \frac{1}{Z}\sum_{\ell\geq2}\frac{1}{\ell\,(\log \ell)^2} \;<\; \infty.
\]
Yet the expected loss of $q$ under $w$ is infinite. Since $q$ is laconic, there are $A > 0$, $B$, and $N$ with $-\log q_k \geq A\,\ell_k^2 + B$ for all $k\geq N$; since $\ell_k \to \infty$, we may enlarge $N$ so that also $A\,\ell_k^2/2 \geq -B$, and hence $-\log q_k \geq A\,\ell_k^2/2$, for all $k \geq N$. Every text longer than $x_N$ has index above $N$, and all terms are nonnegative, so, by Bertrand's criterion~\eqref{eq-bertrand} with $\alpha = 0$ and $\beta = 2$,
\[
\sum_k w_k\,(-\log q_k) \;\geq\; \sum_{k\geq N} w_k\,\frac{A\,\ell_k^2}{2} \;\geq\; \sum_{\ell>\ell_N}\frac{1}{Z\,\ell^2(\log \ell)^2}\cdot\frac{A\,\ell^2}{2} \;=\; \frac{A}{2Z}\sum_{\ell>\ell_N}\frac{1}{(\log \ell)^2} \;=\; \infty.
\]

\emph{The perturbation.} We may assume $\epsilon \leq 1$, since the total variation distance never exceeds $1$. Let $\epsilon'' := \min\{\epsilon/2,\ \eta/\bar\ell(w)\}$, which lies in $(0,\frac12]$, and move $p$ slightly toward $w$:
\[
p'' \;:=\; (1-\epsilon'')\,p + \epsilon''\,w.
\]
This is a distribution, and $p''_k \geq \epsilon'' w_k$ for every $k$. Since $p'' - p = \epsilon''\,(w - p)$, its distance from $p$ is
\[
\dTV(p,p'') \;=\; \epsilon''\,\dTV(p,w) \;\leq\; \epsilon'' \;\leq\; \frac{\epsilon}{2} \;<\; \epsilon,
\]
and its expected length grows by at most $\epsilon''\,\bar\ell(w)$:
\[
\bar\ell(p'') \;=\; (1-\epsilon'')\,\bar\ell(p) + \epsilon''\,\bar\ell(w) \;\leq\; \bar\ell(p) + \epsilon''\,\bar\ell(w) \;\leq\; \bar\ell(p) + \eta.
\]
Finally, $\E_{p''}[-\log q(X)] = \sum_k p''_k\,(-\log q_k) \geq \epsilon''\sum_k w_k\,(-\log q_k) = \infty$.
\end{proof}

The two lemmas display the whole mechanism. Expected length charges a text of length $\ell$ only $\ell$, whereas a laconic model eventually charges it a loss of order $\ell^2$; so a tail can be cheap in expected length yet ruinous in expected loss.

\subsection{Density}\label{app-density}

\begin{proof}[Proof of Lemma~\ref{lem-density}$(b)$]
The natural strategy is to apply the two perturbations of Appendix~\ref{app-steep} directly: Lemma~\ref{lem-finite} will bring finite risk arbitrarily close, and Lemma~\ref{lem-infinite} infinite risk. The only obstacle is a distribution on the boundary of the budget, with $\bar\ell(p) = c$, where there is no room left to add a heavy tail without exceeding the budget. We will get around it by first shortening $p$ slightly, to make room. Here are the details.

Let $q$ be laconic, $c > 1$, $p \in \Pcal_{\,\bar\ell\le c}$, and $\epsilon > 0$. On $\Pcal_{\,\bar\ell\le c}$ we have $1 \leq \bar\ell(p) \leq c$, so $R(p,q) = \infty$ iff $\E_p[-\log q(X)] = \infty$.

\emph{Finite risk nearby.} Lemma~\ref{lem-finite} gives $p'$ within $\epsilon$ of $p$ with $\bar\ell(p') \leq \bar\ell(p) \leq c$ and $\E_{p'}[-\log q(X)] < \infty$; so $p' \in \Pcal_{\,\bar\ell\le c}$ and $R(p',q) < \infty$.

\emph{Infinite risk nearby, inside the budget.} If $\bar\ell(p) < c$, Lemma~\ref{lem-infinite} with $\eta := c - \bar\ell(p)$ gives $p''$ within $\epsilon$ of $p$ with $\bar\ell(p'') \leq c$ and $\E_{p''}[-\log q(X)] = \infty$; so $p'' \in \Pcal_{\,\bar\ell\le c}$ and $R(p'',q) = \infty$.

\emph{Infinite risk nearby, on the boundary.} If $\bar\ell(p) = c$, then, since $c > 1$, $p$ charges some text $x_i$ of length $\ell_i \geq 2$. Move a little of its mass to the shortest text: with $\epsilon' := \min\{p_i,\epsilon/2\}$, let
\[
\tilde p \;:=\; p - \epsilon'\,\delta_{x_i} + \epsilon'\,\delta_{x_1}, \qquad\text{so that}\qquad \dTV(p,\tilde p) = \epsilon' \leq \frac{\epsilon}{2} \quad\text{and}\quad \bar\ell(\tilde p) = c - \epsilon'\,(\ell_i - 1) < c.
\]
Now there is room: Lemma~\ref{lem-infinite} at $\tilde p$, with radius $\epsilon/2$ and $\eta := c - \bar\ell(\tilde p)$, gives $p''$ with $\dTV(\tilde p,p'') < \epsilon/2$, $\bar\ell(p'') \leq c$, and infinite risk. By the triangle inequality, $p''$ lies within $\epsilon$ of $p$.
\end{proof}

\begin{appremark}[Density in the parameter space]
The same perturbations show that, in $\Pcal_{\,\bar\ell<\infty}\times\Qcal$, the states of finite risk and the states of infinite risk are each dense, and both remain dense when $\Qcal$ is replaced by $\Qfull$; this is the sense in which, as Section~\ref{sec-intro} puts it, situations with infinite risk lie arbitrarily close to every situation. Given a nonempty open set $U\times V$ of states, pick a laconic $q \in V$ (Lemma~\ref{lem-steep-dense}) and $p^\circ \in U$ with $B_\epsilon(p^\circ)\cap\Pcal_{\,\bar\ell<\infty} \subseteq U$, take any $c > \max\{1, \bar\ell(p^\circ)\}$, and apply Lemma~\ref{lem-density}$(b)$ in $\Pcal_{\,\bar\ell\le c} \subseteq \Pcal_{\,\bar\ell<\infty}$.
\end{appremark}

\subsection{The Vacillation Lemma}\label{app-vacillation}

\begin{proof}[Proof of Lemma~\ref{lem-vacillation}]
It helps to picture a game between nature and the test. Nature will move the true distribution by ever smaller steps, each time to the other side of the hypothesis, and after each move it will wait until the test has switched its verdict accordingly. The limit of these moves will be a distribution at which the test switches its verdict infinitely often. Two facts make the game work. Each answer is dense, so nature can always cross to the other side in a step as small as it likes. And the test's verdict probabilities move only slightly when the distribution does (Lemma~\ref{lem-continuity}), so a switch, once made, survives all of nature's later and smaller steps. The proof below makes this picture precise.

Let $\phi$ be a test of $H$, and write $a_n(p) := \PP_{p,n}(\phi_n = \mathtt{yes})$. (Read $a$ as \emph{acceptance}: $a_n(p)$ is the probability, at sample size $n$ under distribution $p$, that the test returns \texttt{yes} and thus accepts the hypothesis.) Let $p^\circ \in \Pcal'$ and $\epsilon > 0$. It suffices to find a point of $\Pcal'\cap B_\epsilon(p^\circ)$ at which $\phi$ is inconsistent. Suppose, for reductio, that $\phi$ is consistent throughout $\Pcal'\cap B_\epsilon(p^\circ)$.

\emph{Stage $0$.} As $H$ is dense in $\Pcal'$, pick $p^{(0)} \in H \cap B_{\epsilon/2}(p^\circ)$. By the reductio hypothesis, $a_n(p^{(0)}) \to 1$; choose $n^{(0)}$ with $a_{n^{(0)}}(p^{(0)}) > 0.91$, and let $\epsilon^{(0)} := \min\{\epsilon/4,\ 0.01/n^{(0)}\}$. By Lemma~\ref{lem-continuity}, moving the distribution by less than $\epsilon^{(0)}$ changes $a_{n^{(0)}}$ by less than $n^{(0)}\epsilon^{(0)} \leq 0.01$, so $a_{n^{(0)}} > 0.9$ throughout $B_{\epsilon^{(0)}}(p^{(0)})$: the verdict $\mathtt{yes}$ at sample size $n^{(0)}$ is locked in on that ball.

\emph{Stage $k \to k+1$.} Given $p^{(k)} \in \Pcal'$, $n^{(k)}$, and $\epsilon^{(k)}$, use the density of $\Pcal'\setminus H$ (if $k$ is even) or of $H$ (if $k$ is odd) in $\Pcal'$ to pick $p^{(k+1)}$ in it with $\dTV(p^{(k)},p^{(k+1)}) < \epsilon^{(k)}/2$. Then $p^{(k+1)} \in \Pcal'$, and $p^{(k+1)}$ lies within $\epsilon$ of $p^\circ$. Indeed, $\epsilon^{(j)} \leq 4^{-j}\epsilon^{(0)} \leq 4^{-j}\epsilon/4$ for every $j \leq k$, by the definitions of the radii, so the triangle inequality, applied to the steps already taken, gives
\[
\dTV\big(p^\circ,p^{(k+1)}\big) \;\leq\; \dTV\big(p^\circ,p^{(0)}\big) + \sum_{j=0}^{k}\dTV\big(p^{(j)},p^{(j+1)}\big) \;<\; \frac{\epsilon}{2} + \sum_{j\geq0}\frac{4^{-j}\epsilon}{8} \;=\; \frac{\epsilon}{2} + \frac{\epsilon}{6} \;<\; \epsilon.
\]
So, by the reductio hypothesis, $\phi$ is consistent at $p^{(k+1)}$. Choose $n^{(k+1)} > n^{(k)}$ with $a_{n^{(k+1)}}(p^{(k+1)}) < 0.09$ (if $k$ is even) or $> 0.91$ (if $k$ is odd), and let $\epsilon^{(k+1)} := \min\{\epsilon^{(k)}/4,\ 0.01/n^{(k+1)}\}$. As in stage $0$, Lemma~\ref{lem-continuity} then locks in the new verdict: $a_{n^{(k+1)}} < 0.1$, respectively $> 0.9$, throughout $B_{\epsilon^{(k+1)}}(p^{(k+1)})$.

\emph{Nesting.} Since $\dTV(p^{(k)},p^{(k+1)}) < \epsilon^{(k)}/2$ and $\epsilon^{(k+1)} \leq \epsilon^{(k)}/4$, every point within $\epsilon^{(k+1)}$ of $p^{(k+1)}$ lies within $\epsilon^{(k)}/4 + \epsilon^{(k)}/2 < \epsilon^{(k)}$ of $p^{(k)}$. So each closed ball $\{p : \dTV(p,p^{(k+1)}) \leq \epsilon^{(k+1)}\}$ lies inside the open ball $B_{\epsilon^{(k)}}(p^{(k)})$, and all of them lie inside $B_{\epsilon^{(0)}}(p^{(0)}) \subseteq B_\epsilon(p^\circ)$, since $\epsilon^{(0)} \leq \epsilon/4$ and $p^{(0)}$ lies within $\epsilon/2$ of $p^\circ$.

\emph{The limit.} The steps satisfy $\dTV(p^{(k)},p^{(k+1)}) < \epsilon^{(k)}/2 \leq 4^{-k}\epsilon^{(0)}/2$, which is summable, so $p^{(0)},p^{(1)},\ldots$ is a Cauchy sequence; by completeness of $\Pcal$ it converges to some $p^*$, and $p^* \in \Pcal'$ because $\Pcal'$ is closed. For each $k$, the points $p^{(j)}$ with $j > k$ all lie in the closed ball of radius $\epsilon^{(k+1)}$ about $p^{(k+1)}$, so their limit $p^*$ does too, and hence $p^* \in B_{\epsilon^{(k)}}(p^{(k)})$---a ball on which the verdict at sample size $n^{(k)}$ is locked in. Therefore
\[
a_{n^{(k)}}(p^*) \;>\; 0.9 \ \text{ for even } k, \qquad a_{n^{(k)}}(p^*) \;<\; 0.1 \ \text{ for odd } k,
\]
and since the $n^{(k)}$ increase, $a_n(p^*)$ has no limit as $n\to\infty$. So $\phi$ is inconsistent at $p^* \in \Pcal'\cap B_\epsilon(p^\circ)$---a contradiction.
\end{proof}

\begin{appremark}[Relation to Genin and Kelly]
Lemma~\ref{lem-vacillation} also follows from known results, though nonconstructively. Genin and Kelly (2017, Theorem~4.3) characterize, for an arbitrary set $W$ of Borel probability measures on a metrizable sample space with a countable basis of almost surely clopen sets, the problems that are solvable in chance---for a yes-or-no question, those admitting a pointwise consistent test: each answer must be a countable union of closed sets in the weak topology. Our sample space $\X$ is countable and discrete, so the singletons form such a basis, every test is feasible in their sense, and the weak topology on $\Pcal$ is the total variation topology. Now let $W$ be any nonempty relatively open subset of $\Pcal'$. If some test were consistent throughout $W$, then $H\cap W$ and $W\setminus H$ would both be $F_\sigma$ in $W$; each would then also be $G_\delta$, being the complement of the other; and both would be dense in $W$. But $W$, an open subset of the complete metric space $\Pcal'$, is a Baire space, in which two disjoint dense $G_\delta$ sets cannot exist. The direct proof above shows, in addition, how a test is fooled.
\end{appremark}

\begin{appremark}[Closedness cannot be dropped]
Closedness of $\Pcal'$ is what keeps nature's limit inside the class under consideration, and without it the lemma fails. Let $\mathcal{C}$ be the set of finitely supported distributions on $\X$, and let $H$ consist of those whose support has a largest index, in the enumeration of Appendix~\ref{app-prelim}, that is even. Both $H$ and $\mathcal{C}\setminus H$ are dense in $\mathcal{C}$: add a sliver of mass at a far-off text of the required parity. Yet the test ``is the largest index observed so far even?'' is consistent at every $p \in \mathcal{C}$, since a finitely supported $p$ eventually displays its largest support point, with probability tending to one. The game of the proof above fails here because its limit may have infinite support and so leave $\mathcal{C}$. Like $\Pcal_{\,\bar\ell<L}$, the set $\mathcal{C}$ is a countable union of closed, nowhere dense subsets of itself; this is why Theorem~\ref{thm-main} is proved through the closed sets $\Pcal_{\,\bar\ell\le c}$ rather than on $\Pcal_{\,\bar\ell<L}$ directly.
\end{appremark}

\subsection{Proof of Theorem~\ref{thm-main}}\label{app-proof-main}

\begin{proof}
The proof of part $(b)$ will be developed by zooming in on an arbitrary open set of states. Inside it, we will find a laconic model and a closed budget class $\Pcal_{\,\bar\ell\le c}$, on which the density lemma and the vacillation lemma apply. The vacillation lemma will then yield a distribution at which the test induced by the estimator fails, and the reduction lemma will turn that failing test into a failing estimator. That is the idea. Let us begin.

$(b)$ Let $U\times V$ be a nonempty open subset of $\Pcal_{\,\bar\ell<L}\times\Qfull$, with $U$ and $V$ relatively open; such sets form a base, so it suffices to find a state in $U\times V$ at which $\hat R$ is inconsistent. By Lemma~\ref{lem-density}$(a)$, proved as Lemma~\ref{lem-steep-dense}, there is a laconic $q \in V$. Pick $p^\circ \in U$ and $\epsilon > 0$ with $B_\epsilon(p^\circ)\cap\Pcal_{\,\bar\ell<L} \subseteq U$, and choose $c$ with $\bar\ell(p^\circ) < c < L$. Then $c > 1$, $p^\circ \in \Pcal_{\,\bar\ell\le c}$, and $\Pcal_{\,\bar\ell\le c} \subseteq \Pcal_{\,\bar\ell<L}$; and $\Pcal_{\,\bar\ell\le c}$ is closed, by Lemma~\ref{lem-lsc}. Let
\[
H \;:=\; \{p \in \Pcal_{\,\bar\ell\le c} : R(p,q) < \infty\}.
\]
By Lemma~\ref{lem-density}$(b)$, $H$ and $\Pcal_{\,\bar\ell\le c} \setminus H$ are both dense in $\Pcal_{\,\bar\ell\le c}$. So, by Lemma~\ref{lem-vacillation} with $\Pcal' := \Pcal_{\,\bar\ell\le c}$, the test $\phi$ of Lemma~\ref{lem-reduction}, restricted to $\Pcal_{\,\bar\ell\le c}$, is inconsistent at some $p \in \Pcal_{\,\bar\ell\le c}\cap B_\epsilon(p^\circ)$, a distribution that lies in $U$. By Lemma~\ref{lem-reduction}, $\hat R$ is not consistent at $(p,q)$, a state in $U\times V$.

$(a)$ Since $\Pcal_{\,\bar\ell<L}\times\Qfull$ is nonempty, the dense set of failures given by $(b)$ is nonempty, so $\hat R$ is not pointwise consistent on $\Pcal_{\,\bar\ell<L}\times\Qfull$; nor, therefore, on $\Pcal_{\,\bar\ell<\infty}\times\Qcal$, which contains every state of $\Pcal_{\,\bar\ell<L}\times\Qfull$.
\end{proof}

\begin{appremark}[The weak inequality]
Theorem~\ref{thm-main} holds verbatim with $\Pcal_{\,\bar\ell\le L}$ in place of $\Pcal_{\,\bar\ell<L}$, and the proof is even shorter, because $\Pcal_{\,\bar\ell\le L}$ is itself closed and no detour through a smaller budget is needed. Given a nonempty open set $U\times V$ in $\Pcal_{\,\bar\ell\le L}\times\Qfull$ and a laconic $q \in V$, Lemma~\ref{lem-density}$(b)$ with $c := L$ and Lemma~\ref{lem-vacillation} with $\Pcal' := \Pcal_{\,\bar\ell\le L}$ give a failure in $U\times V$ directly. Since the weak-inequality assumption is decidable in the limit (Appendix~\ref{app-proof-length}), the impossibility does not depend on the length assumption's being untestable.
\end{appremark}

\subsection{Untestability of the Length Assumptions}\label{app-proof-length}

\begin{proof}[Proof of Proposition~\ref{prop-length}]
Both parts will apply the vacillation lemma, so the proof will consist in checking its hypotheses. The main one is that each answer to the question is dense: every distribution will be shown to have, arbitrarily close by, distributions that satisfy the assumption and distributions that violate it. That is the plan; we carry it out part by part.

$(a)$ Apply Lemma~\ref{lem-vacillation} with $\Pcal' := \Pcal$, which is closed, and $H := \Pcal_{\,\bar\ell<\infty}$. \emph{$H$ is dense:} by Lemma~\ref{lem-finite}, every $p$ has finitely supported distributions arbitrarily close by, and those have finite expected length. \emph{Its complement is dense:} the idea is to move $p$ slightly toward a distribution whose tail is too heavy for its expected length to be finite. Let $w'$ be the distribution with
\[
w'_k \;:=\; \begin{cases} \dfrac{1}{Z''\,|\Vcal_0|^{\ell_k-1}\,\ell_k^2\log \ell_k} & \text{if } \ell_k \geq 2,\\[8pt] 0 & \text{if } \ell_k = 1,\end{cases}
\]
where the normalizing constant $Z'' := \sum_{\ell\geq2} 1/(\ell^2\log \ell)$ is finite by Bertrand's criterion~\eqref{eq-bertrand} with $\alpha = 2$ and $\beta = 1$. The texts of length $\ell \geq 2$ jointly receive probability mass $1/(Z''\ell^2\log\ell)$, so the expected length of $w'$ is infinite, by Bertrand's criterion~\eqref{eq-bertrand} with $\alpha = 1$ and $\beta = 1$:
\[
\bar\ell(w') \;=\; \frac{1}{Z''}\sum_{\ell\geq2}\frac{1}{\ell\log\ell} \;=\; \infty.
\]
Given $p \in \Pcal$ and $\epsilon'' \in (0,1]$, the perturbation $p'' := (1-\epsilon'')\,p + \epsilon''\,w'$ lies within $\epsilon''$ of $p$, since $\dTV(p,p'') = \epsilon''\,\dTV(p,w') \leq \epsilon''$, and $\bar\ell(p'') \geq \epsilon''\,\bar\ell(w') = \infty$.

$(b)$ Let $\Pcal' := \Pcal_{\,\bar\ell\le L}$, which is closed by Lemma~\ref{lem-lsc}, and nonempty, and let $H := \Pcal_{\,\bar\ell<L} \subseteq \Pcal'$, so that $\Pcal'\setminus H = \{p : \bar\ell(p) = L\}$. A test of $\Pcal_{\,\bar\ell<L}$ as a hypothesis about $\Pcal$, restricted to $\Pcal'$, is a test of $H$ with the same requirements, so by Lemma~\ref{lem-vacillation} it suffices to show that $H$ and $\Pcal'\setminus H$ are both dense in $\Pcal'$. Let $p \in \Pcal'$ and $\epsilon > 0$.

\emph{$H$ is dense.} If $\bar\ell(p) < L$, then $p \in H$. If $\bar\ell(p) = L > 1$, then $p$ charges some $x_i$ with $\ell_i \geq 2$, and moving mass $\epsilon' \in (0,\min\{p_i,\epsilon\})$ from $x_i$ to the shortest text $x_1$ lowers the expected length by $\epsilon'(\ell_i-1) > 0$, at a cost of $\epsilon' < \epsilon$ in total variation.

\emph{$\Pcal'\setminus H$ is dense.} If $\bar\ell(p) = L$, then $p \in \Pcal'\setminus H$. If $\bar\ell(p) < L$, move $p$ slightly toward a single long text $x$, of length $\ell > L$, just far enough to raise the expected length to exactly $L$:
\[
p' \;:=\; (1-\epsilon')\,p + \epsilon'\,\delta_x, \qquad \epsilon' \;:=\; \frac{L-\bar\ell(p)}{\ell-\bar\ell(p)},
\]
so that $\bar\ell(p') = (1-\epsilon')\,\bar\ell(p) + \epsilon'\,\ell = L$.
Since $\dTV(p,p') \leq \epsilon'$ and $\epsilon' \to 0$ as $\ell\to\infty$, a long enough $x$ puts $p'$ within $\epsilon$ of $p$.
\end{proof}

In part $(b)$, nature's game in the vacillation lemma alternates between distributions with $\bar\ell(p)<L$ and distributions with $\bar\ell(p) = L$ exactly, so the inconsistency may occur at a distribution on the boundary. The strict inequality is essential to part $(b)$. With $\Pcal_{\,\bar\ell\le L}$ as the hypothesis, its complement $\{p : \bar\ell(p) > L\}$ is open, by Lemma~\ref{lem-lsc}, so the hypothesis is closed and its complement open. Both are countable unions of closed sets---an open set is $F_\sigma$ in a metric space---so by the characterization of Genin and Kelly (2017, Theorem~4.3), with $W := \Pcal$, the weak-inequality hypothesis admits a pointwise consistent test: it is decidable in the limit. As the preceding remark shows, this does not affect Theorem~\ref{thm-main}.

% =====================================================================
\section{Proofs for Sections~\ref{sec-models} and~\ref{sec-estimand}}\label{app-proofs-pos}

This appendix proves the positive results. Appendix~\ref{app-floor} shows that a floor gives consistency; Appendices~\ref{app-transformer} and~\ref{app-window} show that Transformers and finite context windows give a floor; Appendix~\ref{app-model-density} shows that only proof or design can certify one; Appendix~\ref{app-smoothing} treats smoothing; and Appendix~\ref{app-capped} proves the results on the capped risk and the bounded error.

\subsection{A Floor Gives Consistency}\label{app-floor}

The following lemma describes what the holdout estimate does at every state, including the states with infinite risk; both Proposition~\ref{prop-floor} and Theorem~\ref{thm-capped} rest on it.

\begin{applemma}\label{lem-holdout-as}
For every state $(p,q) \in \Pcal_{\,\bar\ell<\infty}\times\Qcal$, $\hat R^{\hold}_n(Y;q) \to R(p,q)$ almost surely, in the order topology of $[0,\infty]$, with $\infty/a := \infty$ for finite $a > 0$.
\end{applemma}

\begin{proof}
In outline, the proof has two cases, according to whether $p$ can produce a text to which $q$ assigns probability zero. If it can, the risk is infinite, and the holdout estimate will eventually jump to $\infty$ and stay there. If it cannot, the strong law of large numbers will apply to the numerator and the denominator separately. We take the cases in turn.

Let $A := \{x : q(x) = 0\}$.

\emph{If $p(A) > 0$,} then $R(p,q) = \infty$, because $\E_p[-\log q(X)] \geq p(A)\cdot\infty = \infty$. Almost surely some $Y^{(j)}$ falls in $A$, and from then on the numerator of $\hat R^{\hold}_n$, and hence $\hat R^{\hold}_n$ itself, equals $\infty$.

\emph{If $p(A) = 0$,} the losses $-\log q(Y^{(j)})$ are almost surely finite, nonnegative, and i.i.d.\ with mean $\E_p[-\log q(X)] \in [0,\infty]$, and the lengths $|Y^{(j)}|$ are i.i.d.\ with mean $\bar\ell(p) \in [1,\infty)$. By the strong law of large numbers, which holds for nonnegative summands even when their mean is infinite (Durrett, 2019, Section~2.4), almost surely
\[
\frac1n\sum_{j=1}^n -\log q\big(Y^{(j)}\big) \;\to\; \E_p[-\log q(X)] \qquad\text{and}\qquad \frac1n\sum_{j=1}^n \big|Y^{(j)}\big| \;\to\; \bar\ell(p).
\]
The holdout estimate is the ratio of these two averages, so it converges almost surely to $\E_p[-\log q(X)]/\bar\ell(p) = R(p,q)$; this holds also when the numerator's mean is infinite, since the denominator's limit is finite and at least $1$.
\end{proof}

\begin{proof}[Proof of Proposition~\ref{prop-floor}]
If $q$ has a floor $\beta$, then $R(p,q) \leq \log(1/\beta) < \infty$ for every $p \in \Pcal_{\,\bar\ell<\infty}$, as shown in Section~\ref{sec-models}. By Lemma~\ref{lem-holdout-as}, the holdout estimate converges to $R(p,q)$ almost surely, so $\Error(\hat R^{\hold}_n,R(p,q)) \to 0$ almost surely, and hence in probability.
\end{proof}

Notice that Lemma~\ref{lem-holdout-as} holds at every state, including those with infinite risk. The holdout estimate fails only because, under full support, it converges to $\infty$ through finite values, and so never says that the risk is infinite.

\subsection{Transformers}\label{app-transformer}

We first define the class $\QT$ precisely, and then prove that every model in it has a floor.

\begin{appdefinition}\label{def-transformer}
A model $q$ belongs to $\QT$ iff, for some dimensions and some number $N$ of layers, its conditionals are computed as follows.\par
$(i)$ \emph{Input.} For a context $s = s_1\cdots s_{t-1}$, with $t \geq 1$, let $h^0_1 := e(\eos) + \pi_1$ and $h^0_i := e(s_{i-1}) + \pi_i$ for $2 \leq i \leq t$. Here $e$ embeds $\Vcal$ in $\mathbb{R}^d$---following a common convention, $\eos$ also serves as the start-of-text symbol---and the positional vectors satisfy $\sup_i\|\pi_i\| < \infty$, possibly with $\pi_i \equiv 0$.\par
$(ii)$ \emph{Layers.} For $l = 1,\ldots,N$ and $1 \leq j \leq t$,
\[
h^l_j := F_l\big(h^{l-1}_j,\, a^l_{j,1},\ldots,a^l_{j,M_l}\big), \qquad a^l_{j,m} := \sum_{i=1}^{j}\alpha^l_{j,m,i}\,G_{l,m}\big(h^{l-1}_i\big),
\]
where the maps $F_l$ and $G_{l,m}$ are continuous and do not depend on $t$, and, for each $j$ and $m$, the weights $(\alpha^l_{j,m,i})_{i\leq j}$ form a probability vector, which may depend arbitrarily on $t$, on positions, and on $h^{l-1}_1,\ldots,h^{l-1}_t$.\par
$(iii)$ \emph{Output.} $q(\cdot\mid s) := \mathrm{softmax}(W h^N_t + b)$ over $\Vcal$.
\end{appdefinition}

Clause $(ii)$ covers multi-head causal self-attention---head $m$ with value map $G_{l,m}$ and attention weights $\alpha^l_{j,m,\cdot}$, however queries, keys, and positional information produce them---together with the output projection, residual connections, feedforward blocks with continuous activations, and layer normalization with a positive stabilizer, all absorbed into $F_l$.

\begin{proof}[Proof of Theorem~\ref{thm-transformer}]
Everything turns on one observation: attention only averages. So, however long the context, every internal representation will stay inside one fixed compact set---averaging can never carry a vector outside the convex hull of the vectors averaged. Bounded representations give bounded logits, and bounded logits give a floor. The rest is bookkeeping, which we now carry out.

\emph{Compactness, layer by layer.} Let $K_0$ be the closure of the set of input vectors $\{e(v) + \pi_i : v \in \Vcal,\ i \geq 1\}$, which is bounded, and hence compact by the Heine--Borel theorem. Suppose that every $h^{l-1}_i$ lies in a compact set $K_{l-1}$, whatever the context. Then $G_{l,m}(K_{l-1})$ is compact, as a continuous image of a compact set, and so is its convex hull $C_{l,m}$, by Carath\'eodory's theorem. Every $a^l_{j,m}$ lies in $C_{l,m}$, being a convex combination of points of $G_{l,m}(K_{l-1})$; so every $h^l_j$ lies in the compact set
\[
K_l \;:=\; F_l\big(K_{l-1}\times C_{l,1}\times\cdots\times C_{l,M_l}\big).
\]
By induction on $l$, $h^N_t \in K_N$ for every context, of every length.

\emph{From bounded logits to a floor.} The logits $Wh^N_t + b$ therefore lie in a compact set, so there is a constant $B$ with $|z_v| \leq B$ for every logit $z_v$ and every context. For every context $s$ and token $v$,
\[
q(v\mid s) \;=\; \frac{e^{z_v}}{\sum_{u\in\Vcal} e^{z_u}} \;\geq\; \frac{e^{-B}}{|\Vcal|\,e^{B}} \;=\; \frac{e^{-2B}}{|\Vcal|} \;=:\; \beta.
\]
In particular $q(\eos\mid s) \geq \beta$, so the probability of not halting within $t$ steps is at most $(1-\beta)^t$, which tends to $0$: the conditionals define a tight model, with full support and floor $\beta$. Proposition~\ref{prop-floor} completes the proof.
\end{proof}

This is the argument of Du et al.\ (2023, Lemma~5.8 and the proof of Theorem~5.9), who apply the final step to the coordinate of $\eos$ alone, which is what tightness requires. Their attention step averages the input vectors directly; the value maps $G_{l,m}$ change nothing.

\paragraph{Scope.}
Sinusoidal encodings (Vaswani et al., 2017) have coordinates in $[-1,1]$. RoPE (Su et al., 2024) rotates queries and keys, and ALiBi (Press et al., 2022) adds biases to attention scores; both act only on the weights $\alpha$, which clause~$(ii)$ leaves arbitrary, so both are covered, as are models without positional vectors and sliding-window attention. A learned table of absolute positions defines $\pi_i$ only up to the table's length, and so does not define a model on all of $\X$; extended in any bounded way, for instance by repeating the last entry, it is covered. Output layers that can assign zero probability, such as top-$k$ truncation or sparse alternatives to the softmax, are excluded, and they leave $\Qfull$. Mixture-of-experts layers (Shazeer et al., 2017) route discontinuously, but the proof uses continuity only to keep the image of a bounded set bounded. If continuity of $F_l$ and $G_{l,m}$ is weakened to boundedness on bounded sets, the same induction runs with bounded sets in place of compact ones, and a layer with continuous experts and gating weights in the simplex has that property. We ignore the effects of finite-precision arithmetic throughout.

\subsection{Finite Context}\label{app-window}

\begin{proof}[Proof of Theorem~\ref{thm-window}]
Let $q \in \Qfull$ have window $w$. For every context $s$, $q(\cdot\mid s) = q(\cdot\mid s_w)$, and $s_w$ ranges over the contexts of at most $w$ tokens, of which there are at most $\sum_{j=0}^{w}|\Vcal_0|^j$. So $q$ has only finitely many distinct conditional probabilities $q(v\mid s_w)$, each positive by full support, and the smallest of them is a floor. Proposition~\ref{prop-floor} completes the proof.
\end{proof}

\begin{appremark}[Variable windows]
Suppose instead that, for every context $s$, there is some $w(s) \leq W$, where $W$ is a predetermined finite size, with $q(\cdot\mid s) = q(\cdot\mid s_{w(s)})$: the window may change from context to context. Then every conditional of $q$ is among the conditionals at contexts of at most $W$ tokens, and the same argument gives a floor. Nothing is claimed when the window can grow without bound; that case is covered, for Transformers, by Theorem~\ref{thm-transformer}.
\end{appremark}

\subsection{Density of Floored and Laconic Models}\label{app-model-density}

Section~\ref{sec-models} asserts that only proof or design can certify a floor, because approximate knowledge of a model can neither certify nor exclude one. The following proposition is the formal content of that claim.

\begin{appproposition}\label{prop-model-density}
The set of models with a floor and the set of laconic models are each dense in $\Qfull$. Since laconic models have no floor and models with a floor are not laconic, the complements of both sets are dense in $\Qfull$ as well.
\end{appproposition}

\begin{proof}
\emph{Models with a floor are dense in $\Qfull$, and indeed in $\Qcal$.} It helps to picture the model we will build as a patchwork. We will cut off a tail of $q^\circ$, keeping finitely many texts, and then mix in a little of the uniform model, whose next-token conditionals all equal $1/|\Vcal|$. Away from the finitely many contexts that lead to the kept texts, the mixture will behave exactly like the uniform model; at those finitely many contexts, its conditionals will be positive. A model that is uniform everywhere except at finitely many contexts, and positive there, has a floor. The proof below makes this picture precise.

Let $q^\circ \in \Qcal$ and $\epsilon \in (0,1]$. Choose $N$ such that the tail mass $t := \sum_{k>N} q^\circ_k$ is less than $\epsilon/2$, and let $h$ be obtained from $q^\circ$ by moving that tail mass to the shortest text $x_1$, as in the proof of Lemma~\ref{lem-finite}; then $h$ is finitely supported and $\dTV(q^\circ,h) = t < \epsilon/2$. The uniform model is $u(x) := |\Vcal|^{-|x|}$. It is a distribution on $\X$: grouping texts by length and using $|\Vcal| = |\Vcal_0| + 1$,
\[
\sum_{x\in\X} u(x) \;=\; \sum_{\ell\geq1} |\Vcal_0|^{\ell-1}\,|\Vcal|^{-\ell} \;=\; \frac{1}{|\Vcal|}\sum_{\ell\geq1}\Big(\frac{|\Vcal_0|}{|\Vcal|}\Big)^{\ell-1} \;=\; \frac{1}{|\Vcal|}\cdot\frac{1}{1-|\Vcal_0|/|\Vcal|} \;=\; 1.
\]
For a context $s$, write $q(s\cdot)$ for the probability of the texts beginning with $s$, as in Section~\ref{sec-models}; then $u(s\cdot) = |\Vcal|^{-|s|}$ and $u(sv\cdot) = |\Vcal|^{-|s|-1}$ for every token $v$, so every conditional of $u$ equals $1/|\Vcal|$. Let
\[
q \;:=\; (1-\alpha)\,h + \alpha\,u, \qquad \alpha := \epsilon/2.
\]
Then $\dTV(h,q) = \alpha\,\dTV(h,u) \leq \alpha$, so $\dTV(q^\circ,q) < \epsilon$, and $q$ has full support, since $u$ does. To see that $q$ has a floor, consider two kinds of context. If $h(s\cdot) = 0$, then $h$ also gives probability zero to every text beginning with $sv$, so $q$ agrees with $\alpha u$ there and $q(v\mid s) = u(v\mid s) = 1/|\Vcal|$. If $h(s\cdot) > 0$, then $s$ is a prefix of one of the finitely many texts in the support of $h$, so there are only finitely many such contexts, and at each,
\[
q(v\mid s) \;=\; \frac{q(sv\cdot)}{q(s\cdot)} \;\geq\; \frac{\alpha\,u(sv\cdot)}{q(s\cdot)} \;>\; 0.
\]
The smallest of $1/|\Vcal|$ and these finitely many positive numbers is a floor for $q$.

\emph{Laconic models are dense in $\Qfull$,} by Lemma~\ref{lem-steep-dense}.

\emph{The complements are dense as well.} By Lemma~\ref{lem-steep-dense}, no laconic model has a floor; so the models without a floor include the laconic ones and are dense. Conversely, no model with a floor is laconic, for the same reason; so the models that are not laconic include those with a floor and are dense.
\end{proof}

\begin{appremark}[An asymmetry]
Proposition~\ref{prop-model-density} realizes an asymmetry between the two coordinates of the parameter space. Restrictions of $\Qcal$, such as having a floor, can be proved, although they cannot be certified by approximate knowledge of the model; whereas restrictions of $\Pcal$, such as $\bar\ell(p) < L$, cannot even be tested (Proposition~\ref{prop-length}).
\end{appremark}

\subsection{Smoothing}\label{app-smoothing}

Smoothing is a third route to a floor by design: it mixes every next-token conditional with the uniform distribution on $\Vcal$.

\begin{appproposition}\label{prop-smoothing}
Let $(\tilde q(\cdot\mid s))_{s\in\Vcal_0^*}$ be any family of next-token distributions on $\Vcal$, tight or not, and let $\epsilon \in (0,1]$. Define $q_\epsilon(v\mid s) := (1-\epsilon)\,\tilde q(v\mid s) + \epsilon/|\Vcal|$. Then these conditionals define a tight model $q_\epsilon \in \Qfull$ with finite expected length, and $q_\epsilon$ has floor $\epsilon/|\Vcal|$.
\end{appproposition}

\begin{proof}
Every conditional is at least $\epsilon/|\Vcal|$. In particular, the probability of emitting $\eos$ is at least $\epsilon/|\Vcal|$ at every step, so the probability of not halting within $t$ steps is at most $(1-\epsilon/|\Vcal|)^t$, which tends to $0$ and is summable over $t$. So the conditionals define a distribution on $\X$, and, by the tail-sum formula, its expected length is finite. Its prefix probabilities are products of the conditionals, so its next-token conditionals are the $q_\epsilon(\cdot\mid s)$, each at least $\epsilon/|\Vcal|$.
\end{proof}

\subsection{The Capped Risk and the Bounded Error}\label{app-capped}

\begin{proof}[Proof of Theorem~\ref{thm-capped}]
By Lemma~\ref{lem-holdout-as}, $\hat R^{\hold}_n \to R(p,q)$ almost surely in the order topology of $[0,\infty]$, and the cap $t \mapsto \min\{t,r\}$ is continuous in that topology. So $\min\{\hat R^{\hold}_n,r\} \to \min\{R(p,q),r\}$ almost surely. Both lie in $[0,r]$, so $\Error$ between them is their absolute difference, which tends to $0$ almost surely, and hence in probability.
\end{proof}

The capped risk has an equivalent formulation as a bounded error. Let
\[
d_A(u,v) \;:=\; |\arctan u - \arctan v| \qquad\text{on } [0,\infty], \text{ with } \arctan\infty := \pi/2.
\]
Since $\arctan$ is an order isomorphism of $[0,\infty]$ onto $[0,\pi/2]$, $d_A$ is a bounded metric inducing the order topology, under which $[0,\infty]$ is compact. Unlike $\Error$, it puts large finite values close to $\infty$.

\begin{appproposition}\label{prop-arctan}
Let $U_n$ be random variables with values in $[0,\infty]$, and let $v \in [0,\infty]$. Then $d_A(U_n,v) \to 0$ in probability iff $\min\{U_n,r\} \to \min\{v,r\}$ in probability for every $r \in (0,\infty)$.
\end{appproposition}

\begin{proof}
The idea is, roughly, that capping at every level and measuring on the arctan scale are two ways of discounting large values, and that they discount them in the same way: both treat all sufficiently large values as nearly indistinguishable from $\infty$. Details follow.

($\Rightarrow$) The cap $t\mapsto\min\{t,r\}$ is continuous on the compact metric space $([0,\infty],d_A)$, hence uniformly continuous: for every $\eta > 0$ there is $\delta > 0$ such that $d_A(u,v) < \delta$ implies $|\min\{u,r\}-\min\{v,r\}| < \eta$.

($\Leftarrow$) If $v < \infty$, take $r := v+1$. On the event that $|\min\{U_n,r\} - v| < \eta$, with $\eta \leq \frac12$, we have $\min\{U_n,r\} < r$, so $U_n = \min\{U_n,r\}$, and $d_A(U_n,v) \leq |U_n - v| < \eta$, since $\arctan$ is $1$-Lipschitz. Since $\min\{v,r\} = v$, it follows that $\PP\big(d_A(U_n,v) \geq \eta\big) \leq \PP\big(|\min\{U_n,r\} - \min\{v,r\}| \geq \eta\big) \to 0$. If $v = \infty$, let $\eta \in (0,\pi/2)$ and $r := \tan(\pi/2-\eta) + 1$. Then
\[
\big\{d_A(U_n,\infty) \geq \eta\big\} \;=\; \big\{U_n \leq r-1\big\} \;\subseteq\; \big\{|\min\{U_n,r\} - r| \geq 1\big\},
\]
and the probability of the last event tends to $0$, because $\min\{v,r\} = r$.
\end{proof}

\begin{appcorollary}
The holdout estimator is consistent with respect to $d_A$ at every state of $\Pcal_{\,\bar\ell<\infty}\times\Qcal$.
\end{appcorollary}

\begin{proof}
By Theorem~\ref{thm-capped} and Proposition~\ref{prop-arctan}; alternatively, directly from Lemma~\ref{lem-holdout-as}.
\end{proof}

% =====================================================================
\section{Per-Text Normalization}\label{app-rho-prime}

This appendix shows that the negative results do not depend on our choice of normalization, and that the alternative normalization has a stronger positive side. It comes after Appendices~\ref{app-proofs-neg} and~\ref{app-proofs-pos} because it reuses their arguments. Consider the alternative per-token risk
\[
R'(p,q) \;:=\; \E_p\big[-\log q(X)/|X|\big] \;\in\; [0,\infty],
\]
defined for every $p \in \Pcal$. It averages each text's per-token loss over texts, and is no less natural than $R$, though less often considered; the main text focuses on $R$, which is what practice reports. Dividing each text's loss by its own length tames growth by one power of length. In the single-token example of Section~\ref{sec-setting}, with $q(\ell) \propto e^{-\ell^2}$ and $p(\ell)\propto \ell^{-3}$, the per-text loss is about $\ell$, so $R'(p,q) \approx \sum_\ell \ell^{-3}\cdot \ell < \infty$, although $R(p,q) = \infty$. With the even more laconic $q(\ell) \propto e^{-\ell^3}$, whose log loss grows cubically, the per-text loss is about $\ell^2$, and $R'(p,q) \approx \sum_\ell \ell^{-3}\cdot \ell^2 = \infty$. Accordingly, the role of laconic models is played here by models that are laconic of degree $3$ (Appendix~\ref{app-steep}): full-support models $q$ with $-\log q(x) \geq A\,|x|^3 + B$, for some $A > 0$ and $B$, for all texts $x$ longer than some fixed threshold. An estimator of $R'$ is defined as in Section~\ref{sec-setting}, as are consistency, pointwise consistency, and dense inconsistency, with $R'$ in place of $R$ and $\Pcal$ in place of $\Pcal_{\,\bar\ell<\infty}$.

\begin{apptheorem}\label{thm-rho-prime}
Let $L > 1$, and let $\hat R'$ be any estimator of $R'$. Then $(a)$ $\hat R'$ is not pointwise consistent on $\Pcal\times\Qcal$, and not even on the smaller space $\Pcal_{\,\bar\ell<L}\times\Qfull$; and $(b)$ worse, it is densely inconsistent on $\Pcal_{\,\bar\ell<L}\times\Qfull$.
\end{apptheorem}

\begin{proof}
The argument will follow the proof of Theorem~\ref{thm-main} (Appendix~\ref{app-proof-main}), with one change of substance: since $R'$ divides each text's loss by the text's length, a loss growing quadratically in length now yields a per-text loss growing only linearly, which is too slow. So the models must be even more laconic: laconic of degree $3$ rather than $2$. We indicate only what changes; there are three places.

$(1)$ In the reduction lemma, $R'$ replaces $R$; the argument is unchanged.

$(2)$ Models that are laconic of degree $3$ replace laconic ones. They are dense in $\Qfull$ by the proof of Lemma~\ref{lem-steep-dense}, with the tail $\frac{t}{T}e^{-\ell_k^3}$ in place of $\frac{t}{T}e^{-\ell_k^2}$, where now $T := \sum_{k\geq N} e^{-\ell_k^3}$ is finite by the cubic version of the counting fact~\eqref{eq-counting}.

$(3)$ In Lemma~\ref{lem-infinite}, the distribution $w$ becomes
\[
w_k \;:=\; \begin{cases} \dfrac{1}{Z'\,|\Vcal_0|^{\ell_k-1}\,\ell_k^3\log \ell_k} & \text{if } \ell_k \geq 2,\\[8pt] 0 & \text{if } \ell_k = 1,\end{cases}
\]
where the normalizing constant $Z' := \sum_{\ell\geq2} 1/(\ell^3\log \ell)$ is finite by Bertrand's criterion~\eqref{eq-bertrand} with $\alpha = 3$ and $\beta = 1$. The texts of length $\ell \geq 2$ jointly receive probability mass $1/(Z'\ell^3\log\ell)$, so the expected length of $w$ is
\[
\bar\ell(w) \;=\; \frac{1}{Z'}\sum_{\ell\geq2}\frac{1}{\ell^2\log \ell} \;<\; \infty,
\]
by Bertrand's criterion~\eqref{eq-bertrand} with $\alpha = 2$ and $\beta = 1$. For $q$ laconic of degree $3$, choose $A > 0$ and $N$ with $-\log q_k \geq A\,\ell_k^3/2$ for all $k\geq N$, as in the proof of Lemma~\ref{lem-infinite}. Then, by Bertrand's criterion~\eqref{eq-bertrand} with $\alpha = 1$ and $\beta = 1$,
\[
\sum_k w_k\,\frac{-\log q_k}{\ell_k} \;\geq\; \sum_{\ell>\ell_N}\frac{1}{Z'\,\ell^3\log \ell}\cdot\frac{A\,\ell^3}{2\ell} \;=\; \frac{A}{2Z'}\sum_{\ell>\ell_N}\frac{1}{\ell\log \ell} \;=\; \infty.
\]
So moving $p$ slightly toward $w$ makes $R'$ infinite while raising expected length only slightly, exactly as in Lemma~\ref{lem-infinite}.

The rest of the argument, including the boundary step in the proof of Lemma~\ref{lem-density}$(b)$ and the vacillation lemma, is unchanged. Part $(a)$ follows from $(b)$, since the dense set of failures is nonempty and $\Pcal_{\,\bar\ell<L}\times\Qfull \subseteq \Pcal\times\Qcal$.
\end{proof}

The positive side is stronger for $R'$: under a floor, it needs no assumption about $p$ at all, and the convergence is uniform. For $\beta \in (0,1)$, let $\Qcal_{\beta}$ be the set of models with floor at least $\beta$, and let
\[
\hat R'^{\,\hold}_n(y;q) \;:=\; \frac1n\sum_{j=1}^n \frac{-\log q(y^{(j)})}{|y^{(j)}|}
\]
be the per-text holdout estimator.

\begin{appproposition}\label{prop-rho-prime-positive}
For every $n \geq 1$ and $t > 0$,
\[
\sup_{p\in\Pcal,\ q\in\Qcal_\beta}\ \PP_{p,n}\Big(\big|\hat R'^{\,\hold}_n(Y;q) - R'(p,q)\big| \geq t\Big) \;\leq\; 2\exp\!\big(-2nt^2/\log^2(1/\beta)\big).
\]
Hence the per-text holdout estimator is uniformly consistent on $\Pcal\times\Qcal_\beta$ for each $\beta$, and pointwise consistent on $\Pcal\times\Qfloor$.
\end{appproposition}

\begin{proof}
A floor $\beta$ bounds every next-token loss by $\log(1/\beta)$, so the per-text loss, an average of next-token losses, satisfies $0 \leq -\log q(x)/|x| \leq \log(1/\beta)$ for every $x$. The summands of the estimator are therefore i.i.d.\ and confined to an interval of length $\log(1/\beta)$, with mean $R'(p,q)$, and Hoeffding's inequality (Hoeffding, 1963) gives the bound.
\end{proof}

So under a floor, $R'$ admits uniform guarantees with no assumption on $p$, whereas for $R$ we have established only pointwise consistency, and only on $\Pcal_{\,\bar\ell<\infty}$. The per-token normalization that practice prefers is what brings in the untestable assumption on text length.

% =====================================================================
\section{Other Ways Out}\label{app-ways-out}

The main text discusses two ways out of the impossibility. This appendix surveys the others, organized, as in Section~\ref{sec-intro}, by the component of the problem they relax: the parameter space (Section~\ref{sec-models} and Appendix~\ref{app-bounded}), the estimand (Section~\ref{sec-estimand}) or, equivalently, the error measure (Appendix~\ref{app-capped}), the loss (Appendix~\ref{app-loss}), and the topology (Appendix~\ref{app-topology}). Appendix~\ref{app-truncation} treats two practices that bear on several of these at once.

\subsection{Restricting \texorpdfstring{$\Pcal$}{P}: Bounding Text Length}\label{app-bounded}

Suppose that every $p$ in the parameter space is supported on texts of length at most $N$. Then, for every $q \in \Qfull$, the loss of every text that can occur is at most $M_q := \max_{|x|\leq N} -\log q(x) < \infty$, a maximum over finitely many texts. The numerator and the denominator of the holdout estimator are then averages of i.i.d.\ variables confined to $[0,M_q]$ and $[1,N]$, respectively, so Hoeffding's inequality, applied to each, makes the holdout estimator uniformly consistent over such $p$, for each fixed $q$. Unlike the bound on expected length, this assumption is refutable: if $p$ charges a longer text, a sample eventually displays one. Its cost is that no such $N$ is given by anything in the data; to impose one is to decide in advance that longer texts do not occur.

\subsection{Changing the Loss}\label{app-loss}

\paragraph{Why we keep the straightforward error.}
By Proposition~\ref{prop-arctan}, measuring error on the bounded scale $d_A$ makes the holdout estimator consistent at every state. Why, then, not adopt a bounded error as the standard? Because the unbounded scale of the log score is not incidental: it is forced by two properties that make the log score what it is. A scoring rule $\Loss(q,x)$ evaluates a predictive distribution $q$ against an outcome $x$. It is \emph{local} if it depends on $q$ only through $q(x)$, the probability assigned to what occurs, and \emph{proper} if, for each $p$, the expected score $q \mapsto \E_p[\Loss(q,X)]$ is minimized at $q = p$, so that honest prediction is optimal.

\begin{apptheorem}[Shuford et al., 1966; McCutcheon, 2019]\label{thm-local}
On a finite outcome space with at least three elements, every non-constant proper scoring rule of the form $\Loss(q,x) = f(q(x))$ is $f(t) = -a\log t + b$, with $a > 0$. Shuford et al.\ (1966) assume that $f$ is differentiable; McCutcheon (2019) shows that this assumption can be dropped.
\end{apptheorem}

The theorem extends to the countable space $\X$. If a local rule is proper over all models on $\X$, it is in particular proper over the data distributions and models supported on any three texts; and every value $t \in (0,1)$ is the probability of a text under such a model, so $f$ must be logarithmic on all of $(0,1)$. Bernardo (1979) gives a related characterization of the logarithmic score as the only smooth, local, proper utility for densities.

So any local, proper, non-constant score is unbounded: it must punish a near-zero probability for what occurs without limit. Measuring error on a bounded scale treats an infinite risk as only finitely worse than a large finite risk, and so discards precisely the distinction that the log score's unboundedness draws. To keep the log score while bounding the error is to accept the score's verdicts on finite risks but to discount its verdict on infinite ones. To bound the error coherently, one would instead adopt a bounded score, which by Theorem~\ref{thm-local} cannot be local, proper, and non-constant at once; and for autoregressive models locality is structural, since it is what makes the loss of a text the sum of its next-token losses, and what makes holdout a \emph{sample-evaluation} estimator, one that depends on $q$ only through $q(y^{(1)}),\ldots,q(y^{(n)})$. We therefore keep the straightforward difference, and treat a bounded error, like the capped risk to which it is equivalent, as an explicit change of question. See Gneiting and Raftery (2007) for proper scoring rules in general.

\paragraph{A bounded, non-local score.}
The Brier score (Brier, 1950) shows what giving up locality costs. On texts it is
\[
\Loss_B(q,x) \;:=\; \sum_{y\in\X}\big(q(y) - \mathbf{1}\{y = x\}\big)^2 \;=\; \|q\|_2^2 - 2q(x) + 1,
\]
which lies in $[0,2]$, so its risk $\|q\|_2^2 + 1 - 2\E_p[q(X)]$ is estimable at Hoeffding rates. But not by a sample-evaluation estimator: the term $\|q\|_2^2 = \sum_{y\in\X}q(y)^2$ is a sum over infinitely many texts, global information about $q$ that can be approximated by simulating from $q$, since $\|q\|_2^2 = \E_{X\sim q}[q(X)]$, but cannot be read off the held-out sample. And on $\X$ both $\E_p[q(X)]$ and $\|q\|_2^2$ are minuscule, since probability is spread over astronomically many texts, so every model's Brier risk lies close to $1$: what boundedness buys in estimability, it pays for in resolution.

\paragraph{The divergence.}
The Kullback--Leibler divergence $D(p\|q) = \E_p[-\log q(X)] - H(p)$ is interpretable, but $H(p)$ requires $\log p$, which no sample reveals; one is left with entropy estimation on a countable alphabet (Antos and Kontoyiannis, 2001). Differences of divergences between two models equal differences of cross-entropies, and so are undefined when both risks are infinite.

\subsection{Finer Topologies}\label{app-topology}

Density is preserved under coarsening---a set dense in a finer topology is dense in a coarser one---so an escape from Theorem~\ref{thm-main}$(b)$ needs a topology strictly finer than total variation; part $(a)$ mentions no topology and is unaffected. The weak topology coincides with the total variation topology on a countable $\X$, and the Hellinger distance is topologically equivalent to it. One candidate survives: neighborhoods defined by the Kullback--Leibler divergence, which are smaller than total variation balls, by Pinsker's inequality. The direction of the divergence matters.
\begin{itemize}\setlength{\itemsep}{0pt}
\item With neighborhoods $\{p' : D(p'\|p) < \delta\}$, every neighbor of a finitely supported $p$ is supported where $p$ is, so it has finite risk under a full-support model; infinite risk is no longer dense, and the escape works.
\item With neighborhoods $\{p' : D(p\|p') < \delta\}$, the perturbation of Lemma~\ref{lem-infinite} survives. Since $p''_k \geq (1-\epsilon'')\,p_k$ for every $k$, we have $D(p\|p'') = \sum_k p_k\log(p_k/p''_k) \leq \log\big(1/(1-\epsilon'')\big)$, which tends to $0$ with $\epsilon''$; so infinite risk remains dense, and the escape fails.
\end{itemize}
Even where it works, the escape declares heavier-tailed distributions infinitely far away, which is close to excluding them from the parameter space. In the limiting case of the discrete topology, no two distributions are close, and the question has been abolished rather than answered.

\subsection{Decoding-Time Truncation and Clipping}\label{app-truncation}

\paragraph{Truncation.}
Top-$k$ and nucleus sampling (Holtzman et al., 2020) produce models with zeros. If $p$ charges a text to which $q$ assigns probability zero, the risk is infinite, but this infinity is detectable. The set of such $p$ is open in total variation, and the holdout estimate equals $\infty$ as soon as such a text is observed, which happens with probability tending to one (Lemma~\ref{lem-holdout-as}). So infinite risk caused by zeros does not threaten consistency. The hard case is the one Theorem~\ref{thm-main} isolates: a full-support model that is laconic.

\paragraph{Clipping.}
Suppose an evaluator were to replace each next-token loss by
\[
\min\big\{-\log q(x_t\mid x_{<t}),\ \log(1/\delta)\big\}
\]
for some $\delta > 0$. Then the per-token losses are bounded, so the clipped loss of a text $x$ is at most $|x|\log(1/\delta)$, the clipped per-token risk is at most $\log(1/\delta)$ on $\Pcal_{\,\bar\ell<\infty}$, and the proof of Proposition~\ref{prop-floor} shows that the holdout estimate of the clipped risk is consistent on $\Pcal_{\,\bar\ell<\infty}\times\Qcal$. Clipping per token thus acts like a floor imposed by the evaluator rather than by the model. But it is a change of loss: the clipped next-token loss is local, bounded, and non-constant, so, by Theorem~\ref{thm-local}, it is not proper when $|\Vcal| \geq 3$.

% =====================================================================
\section{Extended Related Work}\label{app-related}

\paragraph{Estimation on countable alphabets.}
Antos and Kontoyiannis (2001) show that plug-in estimates of entropy and related additive functionals are consistent on countable alphabets but admit no universal rates; Wyner and Foster (2003) study lower limits of entropy estimation, and Mossel and Ohannessian (2019) prove the impossibility of learning the missing mass. These estimands depend on the data distribution alone, with no model coordinate. Entropy may be infinite on a countable alphabet, but these works concern ordinary consistency and rates, not the demand, imposed by our extended metric, to decide whether the estimand is finite.

\paragraph{Topology and testability.}
Gold (1967) and Putnam (1965) introduced identification in the limit and the forcing arguments that establish its limits; Kelly (1996) systematized them into topological characterizations of what can be learned in the limit. Dembo and Peres (1994) gave a topological sufficient condition for hypotheses to be distinguishable in the limit; Genin and Kelly (2017) characterized statistical verifiability and decidability in the limit in the weak topology (see also Genin, 2018); Ermakov (2017) studied consistent testing; and Boeken et al.\ (2026) develop topological criteria for tests based on finite-precision measurements, where the weak topology is the appropriate one. On a countable sample space every event has empty boundary, so the weak and total variation topologies coincide, and the distinction disappears. Shah and Peters (2020) show that conditional independence testing is hard because the null is dense in a suitable sense. Bahadur and Savage (1956) established the impossibility of nontrivial tests and confidence intervals for a mean over nonparametric classes, and Donoho (1988) showed that functionals of a density such as the number of modes admit only one-sided inference; the lower semicontinuity of expected length and of expected loss (Lemma~\ref{lem-lsc}) plays an analogous role here.

\paragraph{Scoring rules.}
The log score is singled out among local proper scoring rules by Shuford et al.\ (1966) and Bernardo (1979), with the differentiability assumption removed by McCutcheon (2019); Gneiting and Raftery (2007) survey proper scoring rules, and the Brier score is due to Brier (1950).

\paragraph{Model selection.}
That good per-model risk estimates need not yield a good selection rule is familiar from the comparison of AIC and BIC: Shao (1997) shows that, in linear regression, which criterion is asymptotically loss efficient depends on whether fixed-dimension correct models are among the candidates, and Yang (2005) shows that the strengths of the two cannot in general be combined.

\paragraph{Mathematical uncertainty and causal discovery.}
Kennedy and O'Hagan (2001) model code uncertainty for computer models; probabilistic numerics (Diaconis, 1988; Hennig et al., 2015) places priors on mathematical objects; information-based complexity (Traub et al., 1988) studies worst-case error over classes of problem instances; and Asmussen et al.\ (1992) give an impossibility result in simulation output analysis. In causal discovery, consistency is often required only off a negligible set of parameters, in measure (Spirtes et al., 2000; Meek, 1995) or in category (Lin and Zhang, 2020).

\end{document}